%% file: main.tex
\documentclass{article}

\usepackage[main, final]{Styles/neurips_2026}
\makeatletter
\renewcommand{\@noticestring}{Preprint.}
\makeatother

\usepackage{microtype}
\usepackage{graphicx}
\usepackage{subcaption}
\usepackage{booktabs} % for professional tables

\usepackage{hyperref}

\PassOptionsToPackage{table,dvipsnames}{xcolor}

\input{math_commands.tex}

\usepackage{url}
\usepackage{annotate-equations}
\usepackage{tikz}
\usepackage[skins,listings,most]{tcolorbox}

\usetikzlibrary{positioning,fit,backgrounds,arrows.meta,calc}

\usepackage{amsfonts}       % blackboard math symbols
\usepackage{nicefrac}       % compact symbols for 1/2, etc.
\usepackage{multirow}
\usepackage{threeparttable}
\usepackage{makecell}
\usepackage{colortbl}
\usepackage{amsmath,amscd,amsbsy,amsfonts,latexsym,url,bm,amsthm,amssymb,dsfont}
\usepackage{wrapfig}
\usepackage{caption}
\usepackage{fvextra}
\usepackage{csquotes}
\usepackage{enumitem}
\usepackage{titletoc}

\usepackage{cleveref}
\crefname{section}{$\mathsection$}{$\mathsection\mathsection$}
\Crefname{section}{$\mathsection$}{$\mathsection\mathsection$}

\usepackage{listings}
\usepackage{mdframed}
\usepackage[colorinlistoftodos,prependcaption]{todonotes}
\usepackage{xargs}
\usepackage{BOONDOX-uprscr}
\usepackage{thmtools}
\usepackage{lipsum}    
\usepackage{algorithm}
\usepackage{algorithmic}
\usepackage{float}         % For better float control

\makeatletter
\renewcommand{\section}{%
  \@startsection{section}{1}{\z@}%
                {-1.6ex \@plus -0.35ex \@minus -0.15ex}%
                {0.95ex \@plus 0.2ex \@minus 0.1ex}%
                {\large\bf\raggedright}%
}
\renewcommand{\subsection}{%
  \@startsection{subsection}{2}{\z@}%
                {-1.35ex \@plus -0.3ex \@minus -0.12ex}%
                {0.65ex \@plus 0.15ex}%
                {\normalsize\bf\raggedright}%
}
\renewcommand{\subsubsection}{%
  \@startsection{subsubsection}{3}{\z@}%
                {-1.1ex \@plus -0.25ex \@minus -0.1ex}%
                {0.45ex \@plus 0.12ex}%
                {\normalsize\bf\raggedright}%
}
\renewcommand{\paragraph}{%
  \@startsection{paragraph}{4}{\z@}%
                {0.4ex \@plus 0.12ex \@minus 0.05ex}%
                {-0.6em}%
                {\normalsize\bfseries}%
}
\makeatother
\newtheorem{theorem}{Theorem}
\newtheorem{proposition}[theorem]{Proposition}
\newtheorem{lemma}[theorem]{Lemma}
\newtheorem{corollary}[theorem]{Corollary}
\newtheorem{definition}{Definition}

\tcbset{
  graybox/.style={
    enhanced, breakable,
    colback=Gray!5!white,
    colframe=Gray!50!black,
    arc=2mm, boxrule=0.8pt,
    left=4pt, right=4pt, top=2pt, bottom=2pt,
    before skip=8pt, after skip=8pt
  },
  redbox/.style={
    enhanced, breakable,
    colback=BrickRed!8!white,
    colframe=BrickRed!80!black,
    arc=2mm, boxrule=0.8pt,
    left=4pt, right=4pt, top=2pt, bottom=2pt,
    before skip=8pt, after skip=8pt
  },
  bluebox/.style={
    enhanced, breakable,
    colback=NavyBlue!8!white,
    colframe=NavyBlue!80!black,
    arc=2mm, boxrule=0.8pt,
    left=4pt, right=4pt, top=2pt, bottom=2pt,
    before skip=8pt, after skip=8pt
  }
}

\newcommand{\red}[1]{#1}

\newtcolorbox{learningbox}{
   enhanced,
   colback=Gray!3!white,
   colframe=Gray,
   arc=4mm,
   boxrule=1pt,
   left=4pt,
   right=4pt,
   top=4pt,
   bottom=4pt,
   drop shadow={opacity=0.15, xshift=1.5pt, yshift=-1.5pt},
   fonttitle=\bfseries,
   coltitle=Gray,
   before skip=8pt,
   after skip=15pt
}

\newcommand{\CALL}{\langle\texttt{call}\rangle}
\newcommand{\UNCALL}{\langle/\texttt{call}\rangle}
\newcommand{\ANSWER}{\langle\texttt{return}\rangle}
\newcommand{\UNANSWER}{\langle/\texttt{return}\rangle}
\newcommand{\BOS}{\langle\texttt{bos}\rangle}
\newcommand{\mask}{\Pi}
\newcommand{\maskrec}{\Pi_{\mathrm{rec}}}
\newcommand{\frcm}{g}
\newcommand{\fcot}{f}

\title{Recursive Language Models Generalize Out of Domain}
\author{%
  Chenxiao Yang \qquad
  Zhiyuan Li \qquad
  David McAllester \qquad
  Nathan Srebro \\
  Toyota Technological Institute at Chicago \\
  \texttt{\{chenxiao,zhiyuanli,mcallester,nati\}@ttic.edu}
}

\begin{document}

\maketitle

\begin{abstract}
We study when limiting what a language model can see improves learning. We compare standard CoT, the more general learner that reads the full trace, with recursive language models, which restricts itself by solving each subtask in an isolated context. In-distribution, this generality comes for free: CoT can efficiently simulate the recursive rule, so the IID generalization guarantee changes only by a constant factor, and recursion does not offer much. But out of domain, CoT can fit training by relying on context outside the current subtask, i.e. a shortcut that breaks once those tokens change; recursive context isolation rules out this failure mode. Even though CoT's class still covers the recursive rule, simplicity bias picks the shortcut over the truth. Thus, to go beyond distributional accuracy and truly reason, covering the right rule is not enough---in contrast to classical learning theory.
\end{abstract}

\section{Introduction}
\label{sec:introduction}

A central lesson of machine learning is to favor generality over hand-built specificity~\citep{sutton2019bitter}: use a broad, scalable model class and let data pick the rule. This is statistically cheap in IID learning, where the usual guarantees pay for a simple rule that fits the data, not for every rule the class could express~\citep{rissanen1978modeling,grunwald2007minimum,shalev2014understanding}. This paper asks where that argument stops. The argument is distributional at heart, but language models are increasingly expected to reason---to learn the rule itself, and able to answer correctly on any instance, including ones outside training. This is the out-of-domain (OOD) generalization question.

We make this tension concrete in language modeling. In standard chain-of-thought (CoT) reasoning~\citep{nye2021show,wei2022chain}, every subproblem and every intermediate step is written into one growing sequence. A recursive model (RM)~\citep{yang2025recursive} instead starts a new isolated sequence for each subproblem, runs the same predictor there, and passes only the answer back to the parent. A child can open further subproblems, producing a stack of contexts (\Cref{fig:recursive-example}). The same design principle---isolated context instead of the full history---is widely used in current LLM systems~\citep{anthropic2025contexteng,openai2026orchestration}. The immediate benefit is computational: RM reads a shorter sequence at each step and can handle computations whose flattened CoT trace would exceed the context budget.

But the learnability question is subtler. Setting aside the context-length and compute benefits, does RM learn better because it hides the parts of the complete sequence that are irrelevant to the current subproblem? The answer depends on the kind of generalization. \red{On i.i.d.\ samples from the training generator, RM needs fewer samples to reach high accuracy but CoT catches up with enough data (\Cref{sec:exp-iid-compositional,fig:iid-sample-complexity}); in distribution, CoT is not punished for seeing too much.} Out of domain, the conclusion reverses: CoT collapses on traces only modestly longer than training, while RM stays accurate across much longer traces (\Cref{sec:exp-length-depth,tab:per-bin-generalization}).

In \Cref{sec:id-generalization}, we formalize the in-distribution result via the usual generality-over-specificity argument. CoT is the more general learner: it reads the complete sequence rather than enforcing the context isolation that RM builds in. This generality is cheap in distribution because the recursive rule is compactly contained inside CoT: a CoT Transformer can reconstruct the sequence RM would have seen and apply the same next-token rule, using only constant overhead in depth and parameter count. The standard IID description-length bound therefore changes only by a constant factor, so recursion offers at most a marginal statistical advantage.

OOD is where specificity starts to matter. The same generality that lets CoT contain the recursive rule also lets its attention range over the whole flattened trace at each prediction step. Tokens outside the current subproblem can form easy-to-read patterns---nearby intermediate values, repeated formatting, regularities in where answers appear---that are predictive in the training domain without being part of the subproblem's actual rule. For example, an inner subproblem may return a value that remains visible in the flattened trace and happens to equal the final answer on the training domain, so CoT can learn to copy it (\Cref{sec:exp-case-studies,tab:robustness}). RM's context isolation hides such cues by construction.

\Cref{sec:ood} formalizes this danger via the Minimum Description Length (MDL) view of learning~\citep{rissanen1978modeling,grunwald2007minimum}. The simplicity bias that makes generality cheap in distribution does not distinguish the intended mechanism from a shorter rule over the complete sequence that merely fits training---a shortcut. The desired counterproperty is invariance~\citep{peters2016causal,arjovsky2019invariant}: if the correct answer for a subproblem depends only on the subproblem's own description, then changing the surrounding sequence should not change the prediction. RM satisfies this by construction; CoT can represent the invariant rule, but can also represent shorter non-invariant rules that fit the same training examples. We prove that a CoT model can generalize IID perfectly over complete traces and still fail when asked to solve that subproblem alone---it has learned how the answer is exposed in the trace, not how to solve the subproblem.

This suggests that out-of-domain reasoning needs specificity, not just coverage --- the bitter lesson.

\input{sections/sec2_background}
\input{sections/sec3_setup}
\input{sections/sec4_iid}
\input{sections/sec5_ood}
\input{sections/sec5_experiments}

\section{Conclusion}
\label{sec:conclusion}

\red{Recursive language models can be more sample-efficient in distribution, but they do not give an unbounded IID advantage.
CoT can simulate the recursive rule with little statistical cost and catches up once enough data is available.}
The separation appears out of domain because the broader CoT view can make a wrong rule easier to describe than the intended subproblem rule.
Our bitter lesson for out-of-domain reasoning is that coverage can hurt: seeing more can make the wrong rule easier to learn.
Reasoning beyond the training domain requires more than containing the right rule; it requires the specificity that makes the learner use it.

\section*{Acknowledgments}
This work is supported by DARPA AIQ under Agreement No. HR00112520023, NSF CAREER Award 2544658 and OpenAI Superalignment Fast Grant.

\bibliographystyle{plainnat}
\bibliography{iclr2026_conference}

%%%%%%%%%%%%%%%%%%%%%%%%%%%%%%%%%%%%%%%%%%%%%%%%%%%%%%%%%%%%%%%%%%%%%%%%%%%%%%%
%%%%%%%%%%%%%%%%%%%%%%%%%%%%%%%%%%%%%%%%%%%%%%%%%%%%%%%%%%%%%%%%%%%%%%%%%%%%%%%
% APPENDIX
%%%%%%%%%%%%%%%%%%%%%%%%%%%%%%%%%%%%%%%%%%%%%%%%%%%%%%%%%%%%%%%%%%%%%%%%%%%%%%%
%%%%%%%%%%%%%%%%%%%%%%%%%%%%%%%%%%%%%%%%%%%%%%%%%%%%%%%%%%%%%%%%%%%%%%%%%%%%%%%
\newpage
\appendix

\section*{Appendix Table of Contents}
\startcontents[appendix]
\printcontents[appendix]{l}{1}{\setcounter{tocdepth}{2}}

\input{sections/appendix_related}
\input{sections/appendix_compression}
\input{sections/appendix_transformer}
\input{sections/appendix_simulation}
\input{sections/appendix_iid}
\input{sections/appendix_proofs}
\input{sections/appendix_experiments}

%%%%%%%%%%%%%%%%%%%%%%%%%%%%%%%%%%%%%%%%%%%%%%%%%%%%%%%%%%%%%%%%%%%%%%%%%%%%%%%
% NeurIPS Paper Checklist (commented out for camera-ready draft)
%%%%%%%%%%%%%%%%%%%%%%%%%%%%%%%%%%%%%%%%%%%%%%%%%%%%%%%%%%%%%%%%%%%%%%%%%%%%%%%
% \newpage
% \input{sections/checklist}

\end{document}

%% file: math_commands.tex
\usepackage{amsmath,amsfonts,bm}

\def\eqref#1{equation~\ref{#1}}
\def\1{\bm{1}}

\DeclareMathAlphabet{\mathsfit}{\encodingdefault}{\sfdefault}{m}{sl}
\SetMathAlphabet{\mathsfit}{bold}{\encodingdefault}{\sfdefault}{bx}{n}

\newcommand{\concat}{\mathbin{\Vert}}

%% file: sections/sec2_background.tex
\section{Background: Why Learning Works, and the Sweet Lesson}
\label{sec:id-learning}

A recurring lesson of machine learning is that we often do not need to build the right structure into the model by hand. A broad, scalable function class can be enough, as long as it \emph{covers} a good rule and learning has a bias toward finding a simple one.

The simplest version of this argument is the finite-class bound. A standard IID guarantee for a finite class $\mathcal{F}$ depends on the size of the class only through $\log|\mathcal{F}|$, up to accuracy and confidence factors. Thus if a generic class $\mathcal{F}_{\mathrm{gen}}$ contains an expert class $\mathcal{F}_{\mathrm{expert}}$ and $|\mathcal{F}_{\mathrm{gen}}|=10^6|\mathcal{F}_{\mathrm{expert}}|$, then replacing the expert class by the generic one adds only $\log_2 10^6\approx 20$ to the usual log-class-size term. A broad class can contain many irrelevant predictors without making IID learning impossible, as long as it covers a good rule.

We now move from class size to description length, which lets us compare individual predictors rather than whole classes.

\subsection{Occam's Razor, Simplicity Bias, and Minimum Description Length}

Learning from finite data rests on a preference for simplicity: regularities in the data make it possible to describe the observations more concisely than by listing them. We therefore need a way to assign a description length to each predictor.

\begin{definition}[Complexity Measure]
\label{def:complexity-measure}
A \emph{complexity measure} is a function $\ell$ assigning each predictor $f$ a score $\ell(f)\in\mathbb{R}_{\ge0}\cup\{+\infty\}$ such that $\sum_f 2^{-\ell(f)}\le1$. We write $|f|_\ell:=\ell(f)$.
\end{definition}

The term \emph{description length} is justified by the Kraft condition in the definition: it is the standard prefix-code condition that lets $|f|_\ell$ be read as the number of bits needed to describe $f$. Thus a complexity measure can also be viewed as a \emph{description language}; we use the two terms interchangeably. The key consequence is scarcity: at most $2^k$ predictors can have description length at most $k$ (see \Cref{sec:appendix-compression} for the code interpretation). Thus $\ell$ encodes a built-in preference: before seeing the data, shorter descriptions are treated as more plausible~\citep{shalev2014understanding}.

This definition includes several standard complexity measures. \red{A traditional hypothesis class is recovered by assigning finite length to predictors in the class and length $+\infty$ to predictors outside it. For a finite class $\mathcal{F}$, the uniform choice gives length $\log_2|\mathcal{F}|$ to each predictor in $\mathcal{F}$.} Parameter count gives another example: if a model $f_\theta$ is stored with $b$ bits per parameter, then its description length is proportional to $b|\theta|$. Norm-based measures are also common proxies for simplicity, motivated by explicit regularization and by the implicit bias of gradient descent toward lower-complexity interpolating solutions~\citep{neyshabur2015search,gunasekar2017implicit,soudry2018implicit}. Formally, continuous parameter classes require a discretization or coding scheme; throughout, our bounds apply to the resulting countable set of predictors.

\noindent\textbf{MDL learning.}
Once predictors have description lengths, the MDL rule is immediate: among all predictors consistent with the observed data, select one with the shortest description. This is the Minimum Description Length (MDL) learner~\citep{rissanen1978modeling,grunwald2007minimum}.

Let $\Omega$ be the input space and let $\mathcal{X}\subseteq\Omega$ be the observed inputs. For two predictors $f_1,f_2$, write $f_1\equiv_{\mathcal{X}} f_2$ if $f_1(\mathbf{x})=f_2(\mathbf{x})$ for all $\mathbf{x}\in\mathcal{X}$.

\begin{definition}[MDL Learner]
\label{def:mdl-learner}
\red{Given a complexity measure $\ell$, target $f^*$, and observed inputs $\mathcal{X}$, the \emph{MDL learner} returns}
\begin{align}
    \mathsf{MDL}_\ell(\mathcal{X},f^*)=\arg\min_{f:\,f\equiv_{\mathcal{X}} f^*}|f|_\ell .
\end{align}
\end{definition}

\red{We assume there is at least one finite-length predictor consistent with the observations; in the realizable setting below, this holds because $|f^*|_\ell<+\infty$.} If several predictors have the same shortest length, we use a fixed deterministic tie-breaking rule; equivalently, one can refine $\ell$ once and for all so that the MDL output is unique. We denote this output by $\hat f$.

\begin{theorem}[Realizable MDL Generalization Bound]
\label{thm:mdl-gen-bound}
\red{Fix a complexity measure $\ell$, a target $f^*$ with $|f^*|_\ell<+\infty$, and $\delta\in(0,1)$.} Let $\mathcal{X}=\{\mathbf{x}^{(1)},\ldots,\mathbf{x}^{(n)}\}$ be $n$ i.i.d.\ samples from $\mathcal{D}$ on $\Omega$, and write $\mathcal{L}_{\mathcal{D}}(h)=\Pr_{\mathbf{x}\sim\mathcal{D}}[h(\mathbf{x})\neq f^*(\mathbf{x})]$. \red{For $\hat f=\mathsf{MDL}_\ell(\mathcal{X},f^*)$,} with probability at least $1-\delta$,
\begin{align}
\label{eq:mdl-bound-interp}
\mathcal{L}_{\mathcal{D}}(\hat f)
&\le \frac{|\hat f|_\ell\ln2+\ln(1/\delta)}{n}
\le \frac{|f^*|_\ell\ln2+\ln(1/\delta)}{n}.
\end{align}
\end{theorem}

The proof is the standard realizable Occam argument, included in \Cref{sec:proof-mdl-gen}: a wrong predictor is unlikely to fit all samples, and the Kraft inequality lets us union bound over predictors. The second inequality connects the bound to the target rule. Since the true target $f^*$ is itself a \red{finite-length consistent candidate}, MDL returns a predictor no longer than $f^*$. Thus IID learning pays for a short consistent rule, not for every rule \red{with finite description length}. If the target has a $k$-bit description, then $O(k/\epsilon)$ samples suffice for IID error $\epsilon$, up to confidence terms.

\subsection{Sweet Lesson: Coverage over Specificity in Distribution}
\label{sec:sweet-lesson}
\label{sec:id-generalization-background}

Think of $\ell_A$ as a specialized description language and $\ell_B$ as a more general one. The general language may describe many additional predictors, including irrelevant ones. For IID prediction, this extra coverage is harmless if every rule that was concise under $\ell_A$ is still concise under $\ell_B$.

We say that $\ell_B$ \emph{dominates} $\ell_A$ with constant factor $C\ge1$ (written $\ell_A\preceq\ell_B$) if
\begin{align}
\label{eq:language-domination}
    |f|_{\ell_B}\le C|f|_{\ell_A}
\end{align}
for every predictor $f$ with finite $\ell_A$-length. This is the description-length version of moving from an expert class to a generic function approximator: the generic language may express more functions, but it does not make the expert solutions much longer to describe. For parameter-count codes, this condition holds when every specialized predictor can be implemented in the general architecture with at most a constant-factor increase in parameters.

\begin{corollary}[Expressiveness Is Cheap for IID Prediction]
\label{cor:id-simulation}
\label{cor:expressiveness-free}
\red{If $\ell_A\preceq\ell_B$ and $\hat f_A=\mathsf{MDL}_{\ell_A}(\mathcal{X},f^*)$, $\hat f_B=\mathsf{MDL}_{\ell_B}(\mathcal{X},f^*)$, then}
\begin{align}
\label{eq:length-bound}
|\hat f_B|_{\ell_B}\le C|\hat f_A|_{\ell_A}.
\end{align}
Hence the IID upper bound under the broader language changes only by a constant factor in the description-length term, and so does the corresponding sample-complexity upper bound, up to confidence terms.
\end{corollary}

Indeed, $\hat f_A$ fits the same training data, so it is a feasible candidate for the $\ell_B$-MDL problem. Domination gives $|\hat f_A|_{\ell_B}\le C|\hat f_A|_{\ell_A}$, and the $\ell_B$-MDL output can only be no longer than this candidate.

\noindent\textbf{Sweet lesson.}
This is the sweet side of Sutton's bitter lesson~\citep{sutton2019bitter}: for IID prediction, coverage is more important than perfect specificity. A more general language is not penalized merely because it can express more functions; the IID guarantee pays for the compact rule that learning returns.

There is one important limitation. The corollary compares risks on the training distribution; it does not say that MDL selects the same rule under the two languages. Two predictors can both fit the training data and receive comparable IID guarantees, while disagreeing elsewhere. That difference is invisible to the IID bound, but it is exactly what matters once we ask whether the learned rule works beyond the training distribution.

%% file: sections/sec3_setup.tex
\section{Learning Recursive Language Models}
\label{sec:formal-model}
\label{sec:setup}
\label{sec:modularity-lm}
\label{sec:ar-setup}
\label{sec:recursive-model}
\label{sec:visible-context-setup}

Recursive language models are our test case for this tension. They let the model open a local context for a subproblem, solve it there, and return only the answer to the parent computation. This section sets up that execution model and the learning problems we compare.

\subsection{Autoregressive CoT and Recursive Models}

\noindent\textbf{Autoregressive generation.}
Let $\Sigma$ be a finite vocabulary and set $\Omega := \Sigma^*$. We write $\epsilon$ for the empty sequence and $\concat$ for concatenation. For a parameter value $\theta$, a next-token predictor $f_\theta : \Omega \to \Sigma$ maps a sequence to the next token. Starting from a prompt $\mathbf{x}_0 \in \Omega$, autoregressive generation produces $\mathbf{x}_{t+1} = \mathbf{x}_t \concat f_\theta(\mathbf{x}_t)$, where $\mathbf{x}_t$ is the sequence after $t$ prediction steps. When no time index is shown, $\mathbf{x}$ denotes an arbitrary sequence. In chain-of-thought (CoT) reasoning~\citep{nye2021show,wei2022chain}, the complete sequence $\mathbf{x}_t$ is visible to the predictor.

\noindent\textbf{Recursive generation.}
The \emph{recursive model}~\citep{yang2025recursive} changes what the model reads at each step. When it reaches a subproblem, it starts a new instance of the same predictor, with the same weights but a new isolated context initialized by that subproblem. This child instance may solve the subproblem directly or make further recursive calls. When it returns, only the answer is passed back to the parent context; the child's scratch work is recorded in the complete sequence but is not copied into the parent.

Equivalently, the recursive model maintains a stack of \emph{frames}. The top frame is the \emph{active frame}, and it is the only text shown to the running model instance. Reserve four distinct control tokens $\Sigma_{\mathrm{ctrl}} := \{\CALL,\UNCALL,\ANSWER,\UNANSWER\} \subseteq \Sigma$, and write $\Sigma_{\mathrm{ord}} := \Sigma \setminus \Sigma_{\mathrm{ctrl}}$. A call block $\CALL\,\mathbf{q}\,\UNCALL$ with $\mathbf{q} \in \Sigma_{\mathrm{ord}}^*$ pushes a child frame initialized with $\mathbf{q}$; a return block $\ANSWER\,\mathbf{a}\,\UNANSWER$ with $\mathbf{a} \in \Sigma_{\mathrm{ord}}^*$ pops the child frame and appends $\mathbf{a}$ to its parent. All other tokens append to the top frame.

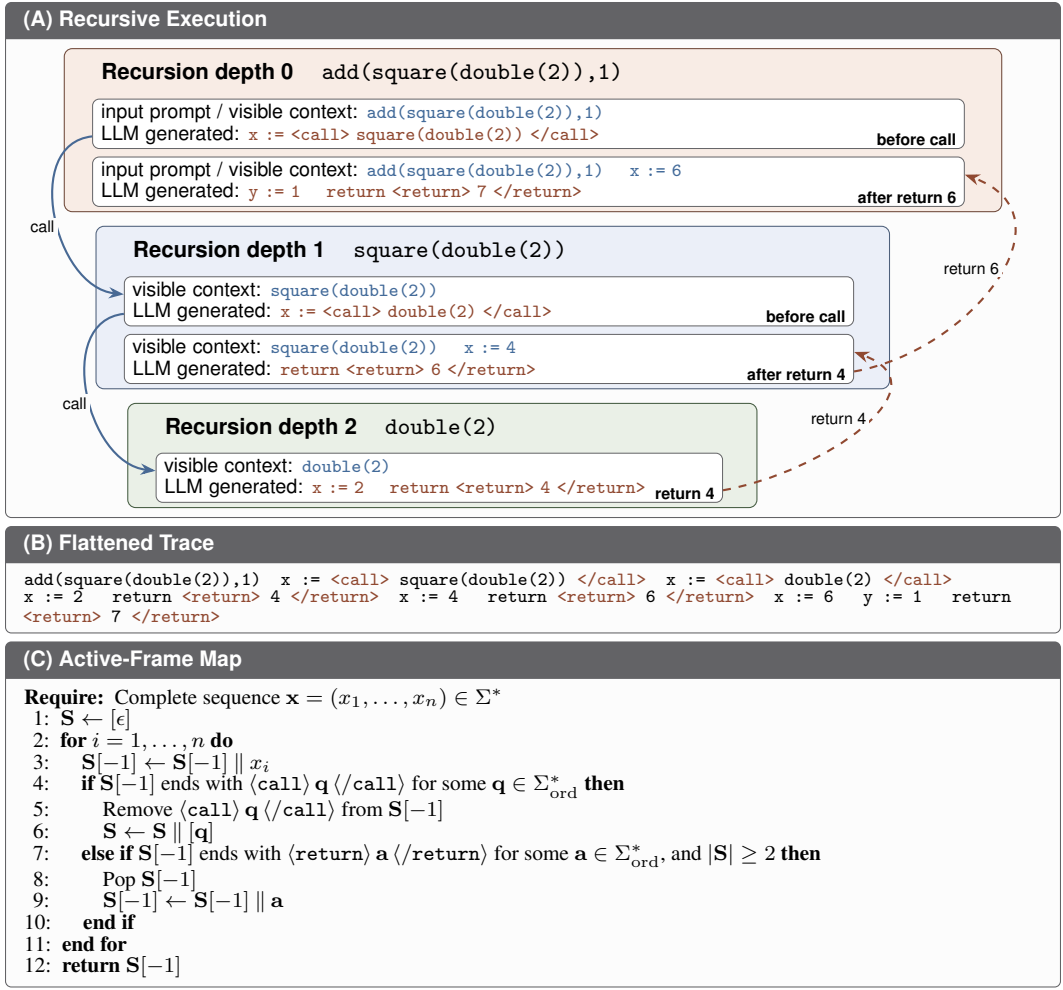
\begin{figure}[t]
\centering
\begin{minipage}{\linewidth}
\centering
\begin{tcolorbox}[
  enhanced, colback=Gray!2, colframe=Gray!50!black, boxrule=0.4pt, arc=2pt,
  title={\bfseries\sffamily (A) Recursive Execution},
  fonttitle=\footnotesize,
  width=\linewidth, left=4pt, right=4pt, top=.5pt, bottom=.5pt, before skip=2pt, after skip=3pt
]
\centering
\begin{tikzpicture}[
  font=\scriptsize,
  frame/.style={rounded corners=3pt, draw=Gray!45!black, fill=Gray!4,
    line width=0.4pt, inner xsep=4.8pt, inner ysep=1.8pt},
  frame0/.style={frame, draw=BrickRed!45!black, fill=BrickRed!7},
  frame1/.style={frame, draw=NavyBlue!45!black, fill=NavyBlue!7},
  frame2/.style={frame, draw=OliveGreen!50!black, fill=OliveGreen!8},
  after/.style={rounded corners=2.5pt, draw=Gray!50!black, fill=white,
    line width=0.35pt, inner xsep=3.2pt, inner ysep=1.75pt, align=left},
  depth/.style={font=\fontsize{8.3}{9.0}\selectfont\sffamily\bfseries, text=black, anchor=west},
  stepbadge/.style={overlay, font=\tiny\sffamily\bfseries, text=black,
    fill=white, inner xsep=1pt, inner ysep=.2pt, anchor=south east},
  callarrow/.style={overlay, -{Stealth[length=2mm]}, thick, NavyBlue!65!black},
  retarrow/.style={overlay, -{Stealth[length=2mm]}, thick, dashed, BrickRed!70!black},
  lab/.style={font=\tiny\sffamily, fill=Gray!2, inner sep=1pt, text=black}
]
\newcommand{\rlmfield}{\fontsize{7.4}{8.0}\selectfont\sffamily}
\newcommand{\rlmtok}{\fontsize{7.4}{8.0}\selectfont\ttfamily}
\newcommand{\rlmexpr}{\fontsize{8.7}{9.0}\selectfont\ttfamily}
\node[depth] (rtitle) at (0,0) {Recursion depth 0 \hspace{1em}{\normalfont\rlmexpr add(square(double(2)),1)}};
\node[after, text width=.84\linewidth, below=.6mm of rtitle.south west, anchor=north west] (rcall) {%
  {\color{black}\rlmfield input prompt / visible context: }{\color{NavyBlue!70!black}\rlmtok add(square(double(2)),1)}\\[0pt]
  {\color{black}\rlmfield LLM generated: }{\color{BrickRed!70!black}\rlmtok x := <call> square(double(2)) </call>}
};
\node[after, text width=.84\linewidth, below=1mm of rcall.south west, anchor=north west] (rafter) {%
  {\color{black}\rlmfield input prompt / visible context: }{\color{NavyBlue!70!black}\rlmtok add(square(double(2)),1) \quad x := 6}\\[0pt]
  {\color{black}\rlmfield LLM generated: }{\color{BrickRed!70!black}\rlmtok y := 1 \quad return <return> 7 </return>}
};

\node[depth, below=3.2mm of rafter.south west, anchor=north west, xshift=.42cm] (stitle) {Recursion depth 1 \hspace{1em}{\normalfont\rlmexpr square(double(2))}};
\node[after, text width=.70\linewidth, below=.6mm of stitle.south west, anchor=north west] (scall) {%
  {\color{black}\rlmfield visible context: }{\color{NavyBlue!70!black}\rlmtok square(double(2))}\\[0pt]
  {\color{black}\rlmfield LLM generated: }{\color{BrickRed!70!black}\rlmtok x := <call> double(2) </call>}
};
\node[after, text width=.70\linewidth, below=1mm of scall.south west, anchor=north west] (safter) {%
  {\color{black}\rlmfield visible context: }{\color{NavyBlue!70!black}\rlmtok square(double(2)) \quad x := 4}\\[0pt]
  {\color{black}\rlmfield LLM generated: }{\color{BrickRed!70!black}\rlmtok return <return> 6 </return>}
};

\node[depth, below=3.2mm of safter.south west, anchor=north west, xshift=.42cm] (dtitle) {Recursion depth 2 \hspace{1em}{\normalfont\rlmexpr double(2)}};
\node[after, text width=.54\linewidth, below=.6mm of dtitle.south west, anchor=north west] (dret) {%
  {\color{black}\rlmfield visible context: }{\color{NavyBlue!70!black}\rlmtok double(2)}\\[0pt]
  {\color{black}\rlmfield LLM generated: }{\color{BrickRed!70!black}\rlmtok x := 2 \quad return <return> 4 </return>}
};

\node[stepbadge] at ([xshift=-.7mm,yshift=.4mm]rcall.south east) {before call};
\node[stepbadge] at ([xshift=-.7mm,yshift=.4mm]rafter.south east) {after return 6};
\node[stepbadge] at ([xshift=-.7mm,yshift=.4mm]scall.south east) {before call};
\node[stepbadge] at ([xshift=-.7mm,yshift=.4mm]safter.south east) {after return 4};
\node[stepbadge] at ([xshift=-.7mm,yshift=.4mm]dret.south east) {return 4};

\coordinate (rwestpad) at ([xshift=-2mm]rcall.west);
\coordinate (reastpad) at ([xshift=3.2mm]rcall.east);
\coordinate (swestpad) at ([xshift=-2mm]scall.west);
\coordinate (seastpad) at ([xshift=3mm]scall.east);
\coordinate (dwestpad) at ([xshift=-2mm]dret.west);
\coordinate (deastpad) at ([xshift=2.8mm]dret.east);

\begin{scope}[on background layer]
  \node[frame0, fit=(rtitle)(rcall)(rafter)(rwestpad)(reastpad)] (root) {};
  \node[frame1, fit=(stitle)(scall)(safter)(swestpad)(seastpad)] (square) {};
  \node[frame2, fit=(dtitle)(dret)(dwestpad)(deastpad)] (double) {};
\end{scope}

\coordinate (rcallgen) at ($(rcall.south west)!0.25!(rcall.north west)$);
\coordinate (scallvis) at ($(scall.south west)!0.64!(scall.north west)$);
\coordinate (scallgen) at ($(scall.south west)!0.25!(scall.north west)$);
\coordinate (dretvis) at ($(dret.south west)!0.64!(dret.north west)$);
\coordinate (dretgen) at ($(dret.south east)!0.25!(dret.north east)$);
\coordinate (saftervis) at ($(safter.south east)!0.64!(safter.north east)$);
\coordinate (saftergen) at ($(safter.south east)!0.25!(safter.north east)$);
\coordinate (raftervis) at ($(rafter.south east)!0.64!(rafter.north east)$);

\draw[callarrow]
  (rcallgen) to[out=186,in=174,looseness=1.12]
  node[lab, left, pos=.55] {call} (scallvis);
\draw[callarrow]
  (scallgen) to[out=186,in=174,looseness=1.12]
  node[lab, left, pos=.55] {call} (dretvis);
\draw[retarrow]
  (dretgen) to[out=10,in=-20,looseness=1.35]
  node[lab, left, pos=.55] {return 4} (saftervis);
\draw[retarrow]
  (saftergen) to[out=10,in=-20,looseness=1.35]
  node[lab, left, pos=.55] {return 6} (raftervis);
\end{tikzpicture}
\end{tcolorbox}

\begin{tcolorbox}[
  enhanced, colback=Gray!2, colframe=Gray!50!black, boxrule=0.4pt, arc=2pt,
  title={\bfseries\sffamily (B) Flattened Trace},
  fonttitle=\footnotesize,
  width=\linewidth, left=4pt, right=4pt, top=.5pt, bottom=.5pt, before skip=2pt, after skip=3pt
]
\noindent\begin{minipage}{.985\linewidth}
{\fontsize{7.1}{7.45}\selectfont\ttfamily\raggedright
add(square(double(2)),1)\quad
x := {\color{BrickRed!70!black}<call>} square(double(2)) {\color{BrickRed!70!black}</call>}\quad
x := {\color{BrickRed!70!black}<call>} double(2) {\color{BrickRed!70!black}</call>}\\[-.35mm]
x := 2 \quad return {\color{BrickRed!70!black}<return>} 4 {\color{BrickRed!70!black}</return>}\quad
x := 4 \quad return {\color{BrickRed!70!black}<return>} 6 {\color{BrickRed!70!black}</return>}\quad
x := 6 \quad y := 1 \quad return {\color{BrickRed!70!black}<return>} 7 {\color{BrickRed!70!black}</return>}
\par}
\end{minipage}
\end{tcolorbox}

\begin{tcolorbox}[
  enhanced, colback=Gray!2, colframe=Gray!50!black, boxrule=0.4pt, arc=2pt,
  title={\bfseries\sffamily (C) Active-Frame Map},
  fonttitle=\footnotesize,
  width=\linewidth, left=4pt, right=4pt, top=.5pt, bottom=.5pt, before skip=2pt, after skip=2pt
]
{\fontsize{8.0}{8.15}\selectfont
\begin{algorithmic}[1]
\REQUIRE Complete sequence $\mathbf{x} = (x_1, \ldots, x_n) \in \Sigma^*$
\STATE $\mathbf{S} \leftarrow [\epsilon]$
\FOR{$i = 1, \ldots, n$}
  \STATE $\mathbf{S}[-1] \leftarrow \mathbf{S}[-1] \concat x_i$
  \IF{$\mathbf{S}[-1]$ ends with $\CALL\,\mathbf{q}\,\UNCALL$ for some $\mathbf{q} \in \Sigma_{\mathrm{ord}}^*$}
    \STATE Remove $\CALL\,\mathbf{q}\,\UNCALL$ from $\mathbf{S}[-1]$
    \STATE $\mathbf{S} \leftarrow \mathbf{S} \concat [\mathbf{q}]$
  \ELSIF{$\mathbf{S}[-1]$ ends with $\ANSWER\,\mathbf{a}\,\UNANSWER$ for some $\mathbf{a} \in \Sigma_{\mathrm{ord}}^*$, and $|\mathbf{S}| \ge 2$}
    \STATE Pop $\mathbf{S}[-1]$
    \STATE $\mathbf{S}[-1] \leftarrow \mathbf{S}[-1] \concat \mathbf{a}$
  \ENDIF
\ENDFOR
\STATE \textbf{return} $\mathbf{S}[-1]$
\end{algorithmic}}
\end{tcolorbox}
\end{minipage}
\caption{\textbf{Recursive expression evaluation.}
The example evaluates $\texttt{add}(\texttt{square}(\texttt{double}(2)),1)$ with arithmetic modulo $10$. The same computation can be viewed either as nested recursive calls or as one flattened trace.}
\label{fig:recursive-example}
\label{fig:obs-function}
\vspace{-0.4em}
\end{figure}

\noindent\textbf{Observation-function view.}
The stack view leaves two objects to keep distinct: the complete sequence accumulated token by token, and the context shown before the next prediction. In \Cref{fig:recursive-example}, panel (B) is the complete sequence, while panel (A) shows the recursive call structure and the frame visible to the predictor. We encode the visible context as a function of the complete sequence. Let $\mask(\mathbf{x})$ be the context shown when the accumulated sequence is $\mathbf{x}$. A visible-context predictor $g_\theta : \Omega \to \Sigma$ generates $\mathbf{x}_{t+1} = \mathbf{x}_t \concat g_\theta(\mask(\mathbf{x}_t))$. The next token is still appended to the complete sequence; $\mask$ only controls what the predictor reads before choosing it.

\begin{definition}[Observation Function]
\label{def:context-mask}
An \emph{observation function} is an idempotent map $\mask: \Omega \rightarrow \Omega$ that returns the context visible to the predictor: $\mask \circ \mask = \mask$. Idempotence says that once a sequence has been reduced to its visible context, observing it again does not reveal additional text.
\end{definition}

The identity map gives ordinary CoT. Other choices of $\mask$ represent other context-selection rules, such as sliding windows, retrieval, or context compaction.

For recursive models, the relevant observation function is the active-frame map $\maskrec$. Given a complete sequence $\mathbf{x}$, $\maskrec$ replays the call and return tokens and returns the top frame; \Cref{fig:recursive-example}(C) gives this replay procedure. Thus $\maskrec(\mathbf{x}_t)$ is exactly the active frame read by the recursive model at step $t$. Since the output is already one frame, $\maskrec\circ\maskrec=\maskrec$.

\subsection{Learning with CoT vs Recursive Models} \label{sec:cot-vs-recursive-learning}

The recursive model gives a natural training recipe: use the active frame as input for each next-token prediction. But the same generated sequence can also be used as ordinary CoT data: ignore the stack and train on the complete sequence directly. The question is whether the active-frame input makes the rule easier to learn, or whether CoT can learn the same next-token behavior from complete sequences.

Write $\frcm^*$ for the target rule on active frames, and $\fcot^*=\frcm^*\circ\maskrec$ for the induced target on complete sequences. Let $\mathcal{X}$ be the complete sequences observed in recursive model runs. For each $\mathbf{x}\in\mathcal{X}$, CoT trains on $(\mathbf{x},\fcot^*(\mathbf{x}))$; the recursive model trains on $(\maskrec(\mathbf{x}),\fcot^*(\mathbf{x}))$, equivalently $(\maskrec(\mathbf{x}),\frcm^*(\maskrec(\mathbf{x})))$. Thus both learners use the same recursive runs and next-token labels; only the input shown to the learner changes. This comparison assumes that the CoT data are flattened recursive-model runs, so the labels on $\mathcal{X}$ satisfy $\fcot^*(\mathbf{x})=\frcm^*(\maskrec(\mathbf{x}))$.

\noindent\textbf{MDL comparison.}
For the MDL comparison, fix an arbitrary observation function $\mask$ and assume the complete-sequence target has the form $\fcot^*=\frcm^*\circ\mask$. We use one base complexity measure $\ell$ for next-token rules; the difference between the learners should come from which part of the sequence the predictor reads, not from changing what counts as a simple rule. CoT applies $\ell$ directly to complete-sequence rules $\fcot$. The recursive model first searches for a local rule $\frcm$ and then uses $\frcm\circ\mask$ on complete sequences. Equivalently, it is standard MDL on complete-sequence rules under the following lifted complexity measure.

\begin{definition}[Lifted Complexity Measure]
\label{def:masked-mdl-learner}
\label{def:lifted-complexity-measure}
\label{def:lifted-description-length}
For a complexity measure $\ell$ and observation function $\mask$, define the lifted complexity measure $\ell^\mask$ by
\begin{align}
\label{eq:lifted-description-length}
    |\fcot|_{\ell^{\mask}}
    :=
    \min_{\frcm:\,\fcot=\frcm\circ\mask} |\frcm|_\ell,
\end{align}
with value $+\infty$ if no such $\frcm$ exists.
\end{definition}

This is again a valid complexity measure: choosing one shortest local representative for each finite-length $\fcot$ gives distinct representatives, so the Kraft sum for $\ell^\mask$ is bounded by the Kraft sum for $\ell$. Under $\ell^\mask$, finite-length complete-sequence rules are exactly those whose predictions have the form $\frcm(\mask(\mathbf{x}))$ for some local rule $\frcm$. Rules that cannot be written this way have length $+\infty$. For the remaining rules, the cost is the shortest local rule before composing with $\mask$. Thus lifted MDL minimizes $|\frcm|_\ell$; this is different from minimizing the original length $|\frcm\circ\mask|_\ell$ of the composed complete-sequence rule.

\begin{proposition}[Local Fitting as Lifted MDL]
\label{prop:lifted-mdl-local}
Let $\mask$ be an observation function, let $\fcot^*=\frcm^*\circ\mask$, and write $\mask(\mathcal{X}):=\{\mask(\mathbf{x}):\mathbf{x}\in\mathcal{X}\}$. If deterministic tie-breaking is chosen consistently on the two minimization problems, then
\begin{align}
\label{eq:local-mdl-lifted-mdl}
    \mathsf{MDL}_{\ell^\mask}(\mathcal{X},\fcot^*)
    =
    \mathsf{MDL}_{\ell}(\mask(\mathcal{X}),\frcm^*)\circ\mask .
\end{align}
\end{proposition}

Thus MDL under the lifted measure is the same as first choosing the shortest local rule on the visible contexts $\mask(\mathcal{X})$, then composing it with $\mask$. Any finite lifted-MDL output has the form $\hat\frcm\circ\mask$; the predictor learned before composition is the local rule $\hat\frcm$.

\noindent\textbf{Notation.}
Below, $\fcot$ and $\hat\fcot$ read complete sequences, while $\frcm$ and $\hat\frcm$ read visible contexts. CoT returns $\hat\fcot:=\mathsf{MDL}_{\ell}(\mathcal{X},\fcot^*)$, and the lifted learner returns $\hat\fcot_{\mask}:=\mathsf{MDL}_{\ell^\mask}(\mathcal{X},\fcot^*)$. By the proposition, $\hat\fcot_{\mask}$ can be unpacked as $\hat\frcm\circ\mask$ for a shortest local rule $\hat\frcm$; for recursive models, $\mask=\maskrec$.

%% file: sections/sec4_iid.tex
\section{IID Generalization: Marginal Benefit of Modeling Recursion}
\label{sec:id-generalization}

When training and test sequences come from the same distribution, does recursion make learning easier simply by restricting the predictor to the visible context rather than the complete sequence? We isolate this IID question by first showing that CoT can simulate the recursive view with constant depth and constant-factor parameter overhead, then translating the simulation into the MDL comparison of \Cref{sec:id-learning}.

\paragraph{Expressiveness.}
Let $\mathcal{A}$ be a Transformer architecture formally defined in \Cref{app:transformer}, with depth $L_\mathcal{A}$ and parameter count $W_\mathcal{A}$.\footnote{We assume causal hard attention~\citep{hahn2020theoretical,merrill2022saturated} and product-gated FFNs~\citep{shazeer2020glu}.} The corresponding next-token predictor class is $\mathcal{F}_\mathcal{A}$.

The recursive model applies a visible-context rule $\frcm$ to the active frame $\maskrec(\mathbf{x})$. A CoT predictor instead reads the complete sequence $\mathbf{x}$. The theorem below formalizes the corresponding simulation: a CoT Transformer can recover the active frame and then apply the same rule.

\begin{theorem}[CoT Simulation of the Recursive Model]
\label{thm:cot-simulation}
For every general Transformer architecture $\mathcal{A}$, there exists another architecture $\mathcal{A}'$ with $L_{\mathcal{A}'}=L_\mathcal{A}+O(1)$ and $W_{\mathcal{A}'}=O(W_\mathcal{A})$. For every choice of visible-context recursive rule $\frcm \in \mathcal{F}_\mathcal{A}$, there exists a CoT simulator $f \in \mathcal{F}_{\mathcal{A}'}$ such that $f(\mathbf{x}) = \frcm(\maskrec(\mathbf{x}))$ for every complete sequence $\mathbf{x}\in\Omega$ that can arise during recursive generation (see \Cref{def:legal-prefix} for the formal definition).
\end{theorem}

The proof is deferred to \Cref{app:proof-cot-simulation}. A key step is to recover the active frame $\maskrec(\mathbf{x})$ from the complete sequence $\mathbf{x}$, which means identifying which call/return blocks are still open. The construction implements this tracking with constant depth and constant-factor learned-parameter overhead. \red{The simulator architecture $\mathcal A'$ is obtained from $\mathcal A$ by a fixed computable construction.} On sequences that cannot arise from the recursive generation process, the simulator may be defined arbitrarily.

\paragraph{CoT dominates recursive models.}
We next specify the description language used for CoT predictors. A finite-precision Transformer realization consists of an architecture $\mathcal A$ and parameters $\theta$. We charge for describing $\mathcal A$, and then charge $b$ bits for each independently stored learned parameter. For a next-token rule $f$, define
\begin{align}
\label{eq:transformer-description-length}
|f|_\ell
:=
\min_{\mathcal A,\theta:\,f_{\mathcal A,\theta}=f}
\left(
|\operatorname{code}(\mathcal A)|+bW_\mathcal A
\right),
\qquad
\min\emptyset:=+\infty.
\end{align}
Here $W_\mathcal A$ is the number of learned scalar parameters in $\mathcal A$. Fix any prefix universal decoder $U$ for finite Transformer architecture specifications, for instance a prefix universal Turing machine whose outputs are parsed as such specifications, and define the architecture code length by $|\operatorname{code}(\mathcal A)|:=\min\{|p|:\,U(p)=\mathcal A\}$.
(Formalized in \Cref{app:proof-iid}.)

For the recursive model, the corresponding description language is the lifted language $\ell^{\maskrec}$ from \Cref{def:lifted-description-length}: first observe the active frame through $\maskrec$, then encode the local Transformer rule. Combining the simulation theorem with the coding convention above gives a constant-factor comparison with the ordinary CoT length.

\begin{corollary}[Comparable Description Length for CoT]
\label{cor:cot-transformer-dominance}
For every complete-sequence predictor $f$ with finite $|f|_{\ell^{\maskrec}}$, there is a CoT predictor $\tilde f$ that agrees with $f$ on the legal recursive-prefix domain considered in \Cref{thm:cot-simulation}, and constants $C,C_0$ independent of the task and training sample such that
\begin{align}
|\tilde f|_\ell\le C\,|f|_{\ell^{\maskrec}}+C_0 .
\end{align}
\end{corollary}

\noindent The comparison is just the same two-part accounting. If a local rule $g$ is realized by $(\mathcal A,\theta)$, the simulator realizes $g\circ\maskrec$ using $(\mathcal A',\theta')$. The fixed architecture translation adds only $O(1)$ code bits, and the learned parameters grow by only a constant factor, so
\begin{align}
|\operatorname{code}(\mathcal A')|+bW_{\mathcal A'}
\le
C\bigl(|\operatorname{code}(\mathcal A)|+bW_\mathcal A\bigr)+C_0.
\end{align}
Thus the CoT description language contains the recursive behavior on legal prefixes up to the simulation-overhead term above.

\paragraph{IID Generalization.}
We now state the corresponding IID guarantee in the learning setup of \Cref{sec:cot-vs-recursive-learning}. Let $\mathcal{D}$ be any distribution over complete sequences. Fix a visible-context recursive rule $\frcm\in\mathcal{F}_\mathcal{A}$, and let $\fcot:=\frcm\circ\maskrec$ be the induced complete-sequence target. For any predictor $h:\Omega\to\Sigma$, measure error against this target by $\mathcal{L}_{\mathcal{D}}(h):=\Pr_{\mathbf{x}\sim\mathcal{D}}[h(\mathbf{x})\neq \fcot(\mathbf{x})]$.

The CoT and recursive learners fit the same next-token labels. They differ only in what their candidate predictors read: CoT reads $\mathbf{x}$, while the recursive model reads $\maskrec(\mathbf{x})$.

\begin{theorem}[IID Generalization Under CoT Simulation]
\label{thm:iid-sample-complexity}
Assume $|\fcot|_{\ell^{\maskrec}}<+\infty$ and that a draw $\mathbf{x}\sim\mathcal{D}$ can arise during recursive generation almost surely. Let $\mathcal{X}\sim\mathcal{D}^m$ be the training set of sequences. \red{Define}
\begin{align}
\hat\fcot:=\mathsf{MDL}_{\ell}(\mathcal{X},\fcot),
\qquad
\hat f_{\maskrec}:=\mathsf{MDL}_{\ell^{\maskrec}}(\mathcal{X},\fcot).
\end{align}
Then, for any $\delta\in(0,1)$, with probability at least $1-\delta$ over the draw of $\mathcal{X}$,
\begin{align}
\mathcal{L}_{\mathcal{D}}(\hat\fcot)
\le
\frac{\bigl(C\,|\fcot|_{\ell^{\maskrec}}+C_0\bigr)\ln 2+\ln(2/\delta)}{m},
\qquad
\mathcal{L}_{\mathcal{D}}(\hat f_{\maskrec})
\le
\frac{|\fcot|_{\ell^{\maskrec}}\ln 2+\ln(2/\delta)}{m}.
\end{align}
\end{theorem}

\noindent The proof is given in \Cref{app:proof-iid}. It applies the realizable MDL bound twice: \red{once under the lifted description language $\ell^{\maskrec}$ and once under the CoT description language $\ell$}. \red{By \Cref{cor:cot-transformer-dominance}, the CoT description-length term is at most a constant-factor simulation overhead larger than the lifted recursive term.} Thus, in this IID comparison, recursion can offer at most a marginal sample-complexity advantage. \red{Equivalently, under the finite-class view from \Cref{sec:id-learning}, $b$-bit discretization gives about $2^{bW_\mathcal A}$ parameter settings for an architecture with $W_\mathcal A$ learned parameters; since the simulator uses $O(W_\mathcal A)$ learned parameters, the corresponding $\log|\mathcal F|$ term changes only by the same constant-factor overhead.}

%% file: sections/sec5_ood.tex
\section{OOD Generalization: Invariance and Risks of Shortcuts}
\label{sec:ood-section}
\label{sec:ood-generalization}
\label{sec:ood}

\subsection{Why OOD Is Different from IID}
\label{sec:ood-mdl-lesson}

IID learning asks whether the selected rule predicts well under the same distribution that produced the training data. Out-of-domain (OOD) generalization asks what happens when the learner is trained on one domain but evaluated on complete sequences beyond that domain. Let $\mathcal{X} \subseteq \Omega$ denote the training set of complete sequences, with target values given by a rule $f^* : \Omega \to \Sigma$. We focus on consistent learners: the learned predictor agrees with $f^*$ on $\mathcal{X}$, and we ask how it behaves beyond $\mathcal{X}$.

\paragraph{Risks of simplicity bias.}
Here simplicity bias plays a different role. In IID learning, a short rule that fits the training data receives a good average-case guarantee under the same distribution. Out of domain, the shortest fitting rule may instead use a pattern specific to the training domain rather than the mechanism we meant to learn. We call such a rule a shortcut: a predictor $\hat r$ with $\hat r \equiv_{\mathcal{X}} f^*$ and $|\hat r|_\ell < |f^*|_\ell$. From the observations alone, $\hat r$ is indistinguishable from the intended rule, so MDL may choose it. The difference appears only when we evaluate on inputs where the training-domain pattern no longer tracks $f^*$.

\paragraph{Two examples.}
Two simple examples isolate the same mechanism; the first two panels of \Cref{fig:shortcut-examples} illustrate them.

\begin{figure}[!t]
\centering
\def\casefont{\fontsize{8}{8.6}\selectfont\sffamily}
\def\casekeyfont{\fontsize{8}{8.6}\selectfont\sffamily\bfseries}
\begin{minipage}{\linewidth}
\centering
\begin{tcolorbox}[
  enhanced, colback=Gray!2, colframe=Gray!50!black, boxrule=0.4pt, arc=2pt,
  title={\bfseries\sffamily (A) Length Generalization Shortcut},
  fonttitle=\footnotesize,
  width=\linewidth, left=4pt, right=4pt, top=.5pt, bottom=.5pt, before skip=2pt, after skip=2pt
]
\centering
\resizebox{.99\linewidth}{!}{%
\begin{tikzpicture}[
  x=1mm, y=1mm,
  frame/.style={rounded corners=3pt, line width=0.45pt, inner xsep=4pt, inner ysep=3pt, align=left},
  shortcut/.style={frame, draw=BrickRed!45!black, fill=BrickRed!7},
  target/.style={frame, draw=OliveGreen!50!black, fill=OliveGreen!8},
  code/.style={rounded corners=2.5pt, draw=Gray!55!black, fill=white,
    line width=0.32pt, inner xsep=4pt, inner ysep=1.9pt, align=left},
  head/.style={font=\casekeyfont, text=black, anchor=north west, inner sep=0pt},
  context/.style={font=\casefont, text=black, anchor=west}
]
\def\codefont{\fontsize{7.65}{7.45}\selectfont\ttfamily}
\node[anchor=north west, text=black] (lookuplabel) at (0,28.6) {%
  {\casekeyfont\color{BrickRed!90!black}Shortcut:} {\casefont lookup table}%
};
\node[code, text width=68mm, anchor=north west] (lookup) at (0,24.0) {%
  {\codefont def shortcut(x):}\\[-.35mm]
  {\codefont \quad return x == "01" or x == "0011"}
};
\node[text=black, text width=76mm, anchor=south west, align=left] at (0,4.0) {%
  {\casekeyfont training:} {\casefont $|x|\le 4$; positives \texttt{01}, \texttt{0011}}\\[1.05mm]
  {\casekeyfont OOD:} {\casefont $|x|>4$; positives \texttt{000111},$\ldots$}
};

\node[anchor=north west, text=black] (rulelabel) at (78,28.6) {%
  {\casekeyfont\color{OliveGreen!85!black}Ground Truth:} {\casefont strings of the form $0^n1^n$}%
};
\node[code, text width=68mm, anchor=north west] (rule) at (78,24.0) {%
  {\codefont def ground\_truth(x):}\\[-.35mm]
  {\codefont \quad i, j = 0, 0}\\[-.35mm]
  {\codefont \quad while i < len(x) and x[i] == "0":}\\[-.35mm]
  {\codefont \quad\quad i += 1}\\[-.35mm]
  {\codefont \quad while j < len(x) and x[len(x)-1-j] == "1":}\\[-.35mm]
  {\codefont \quad\quad j += 1}\\[-.35mm]
  {\codefont \quad return i > 0 and i == j and i + j == len(x)}
};
\end{tikzpicture}}
\end{tcolorbox}

\begin{tcolorbox}[
  enhanced, colback=Gray!2, colframe=Gray!50!black, boxrule=0.4pt, arc=2pt,
  title={\bfseries\sffamily (B) Graph Connectivity Shortcut},
  fonttitle=\footnotesize,
  width=\linewidth, left=4pt, right=4pt, top=.5pt, bottom=.5pt, before skip=2pt, after skip=2pt
]
\centering
\resizebox{.99\linewidth}{!}{%
\begin{tikzpicture}[
  x=1mm, y=1mm,
  head/.style={font=\casekeyfont, text=black, anchor=west},
  context/.style={font=\casefont, text=black, anchor=west},
  graphcard/.style={rounded corners=3pt, draw=Gray!45!black, fill=white,
    line width=0.34pt, inner xsep=3pt, inner ysep=3pt, align=left},
  gnode/.style={circle, draw=Gray!65!black, fill=white, line width=0.64pt,
    minimum size=4.05mm, inner sep=0pt, font=\fontsize{7.35}{7.75}\selectfont\sffamily\bfseries},
  snode/.style={gnode, draw=NavyBlue!65!black, fill=NavyBlue!12, text=NavyBlue!70!black},
  tnode/.style={gnode, draw=BrickRed!65!black, fill=BrickRed!12, text=BrickRed!70!black},
  gedge/.style={draw=Gray!75!black, line width=0.88pt},
  tagtext/.style={inner sep=0pt, font=\fontsize{7.75}{8.25}\selectfont\sffamily,
    anchor=north east, text=black},
  tagcolon/.style={inner sep=0pt, font=\fontsize{7.75}{8.25}\selectfont\sffamily,
    anchor=north west, text=black},
  tagnum/.style={inner sep=0pt, font=\fontsize{7.75}{8.25}\selectfont\sffamily\bfseries,
    anchor=north east}
]
\def\tagok#1{{\color{OliveGreen!78!black}#1}}
\def\tagbad#1{{\color{BrickRed!85!black}#1}}
\node[anchor=base, text=black] at (35,27.6) {{\casekeyfont\color{BrickRed!90!black}Shortcut:} {\casefont $s,t$ connected}};
\node[anchor=base, text=black] at (104,27.6) {{\casekeyfont\color{OliveGreen!85!black}Ground Truth:} {\casefont $s,t$ antipodal on one cycle}};

\node[graphcard, minimum width=46mm, minimum height=22.0mm, anchor=north west] (gpos) at (0,24.6) {};
\node[head] at (1.5,23.05) {training};
\node[head] at (1.5,20.75) {positive};
\draw[gedge] (23,18.55) -- (28.3,15.45) -- (28.3,9.25) -- (23,6.15) -- (17.7,9.25) -- (17.7,15.45) -- cycle;
\node[snode] at (23,18.55) {s};
\node[gnode] at (28.3,15.45) {};
\node[gnode] at (28.3,9.25) {};
\node[tnode] at (23,6.15) {t};
\node[gnode] at (17.7,9.25) {};
\node[gnode] at (17.7,15.45) {};
\node[tagtext] at (41.4,23.05) {shortcut};
\node[tagcolon] at (42.0,23.05) {:};
\node[tagnum] at (44.7,23.05) {\tagok{1}};
\node[tagtext] at (41.4,20.75) {truth};
\node[tagcolon] at (42.0,20.75) {:};
\node[tagnum] at (44.7,20.75) {\tagok{1}};

\node[graphcard, minimum width=46mm, minimum height=22.0mm, anchor=north west] (gneg) at (52,24.6) {};
\node[head] at (52.5,23.05) {training};
\node[head] at (52.5,20.75) {negative};
\draw[gedge] (64,15.8) -- (68.4,12.3) -- (64,8.8) -- (59.6,12.3) -- cycle;
\draw[gedge] (86,15.8) -- (90.4,12.3) -- (86,8.8) -- (81.6,12.3) -- cycle;
\node[snode] at (59.6,12.3) {s};
\node[gnode] at (64,15.8) {};
\node[gnode] at (68.4,12.3) {};
\node[gnode] at (64,8.8) {};
\node[gnode] at (81.6,12.3) {};
\node[gnode] at (86,15.8) {};
\node[gnode] at (90.4,12.3) {};
\node[tnode] at (86,8.8) {t};
\node[tagtext] at (93.4,23.05) {shortcut};
\node[tagcolon] at (94.0,23.05) {:};
\node[tagnum] at (96.7,23.05) {\tagok{0}};
\node[tagtext] at (93.4,20.75) {truth};
\node[tagcolon] at (94.0,20.75) {:};
\node[tagnum] at (96.7,20.75) {\tagok{0}};

\node[graphcard, minimum width=46mm, minimum height=22.0mm, anchor=north west] (good) at (104,24.6) {};
\node[head] at (105.5,23.05) {OOD};
\node[head] at (105.5,20.75) {failure};
\draw[gedge] (127,18.55) -- (132.3,15.45) -- (132.3,9.25) -- (127,6.15) -- (121.7,9.25) -- (121.7,15.45) -- cycle;
\node[snode] at (127,18.55) {s};
\node[tnode] at (132.3,15.45) {t};
\node[gnode] at (132.3,9.25) {};
\node[gnode] at (127,6.15) {};
\node[gnode] at (121.7,9.25) {};
\node[gnode] at (121.7,15.45) {};
\node[tagtext] at (145.4,23.05) {shortcut};
\node[tagcolon] at (146.0,23.05) {:};
\node[tagnum] at (148.7,23.05) {\tagbad{1}};
\node[tagtext] at (145.4,20.75) {truth};
\node[tagcolon] at (146.0,20.75) {:};
\node[tagnum] at (148.7,20.75) {\tagok{0}};
\end{tikzpicture}}
\end{tcolorbox}

\begin{tcolorbox}[
  enhanced, colback=Gray!2, colframe=Gray!50!black, boxrule=0.4pt, arc=2pt,
  title={\bfseries\sffamily (C) Recursive Trace Shortcut},
  fonttitle=\footnotesize,
  width=\linewidth, left=4pt, right=4pt, top=.5pt, bottom=.5pt, before skip=2pt, after skip=2pt
]
\centering
\resizebox{\linewidth}{!}{%
\begin{tikzpicture}[
  x=1mm, y=1mm,
  panel/.style={rounded corners=3pt, draw=Gray!45!black, fill=white,
    line width=0.4pt, inner sep=0pt},
  ptitle/.style={font=\fontsize{8.05}{8.35}\selectfont\sffamily,
    anchor=base west, text=black, inner sep=0pt},
  psubtitle/.style={font=\fontsize{8.05}{8.35}\selectfont\sffamily,
    anchor=base west, text=black, inner sep=0pt},
  line/.style={font=\fontfamily{lmtt}\fontsize{6.85}{6.35}\selectfont,
    anchor=west, text=black, inner sep=0.35pt},
  ann/.style={rounded corners=1pt, draw=Orange!70!black, fill=YellowOrange!28,
    font=\fontsize{5.15}{5.35}\selectfont\sffamily\bfseries,
    anchor=east, inner xsep=0.75pt, inner ysep=0.2pt, text=Orange!85!black},
  hitred/.style={line, rounded corners=1.2pt, draw=BrickRed!38!black,
    fill=BrickRed!8, text=BrickRed!72!black, inner xsep=0.85pt, inner ysep=0.45pt},
  hitgreen/.style={line, rounded corners=1.2pt, draw=OliveGreen!45!black,
    fill=OliveGreen!8, text=OliveGreen!70!black, inner xsep=0.85pt, inner ysep=0.45pt},
  depth/.style={font=\fontfamily{lmtt}\fontsize{6.25}{6.45}\selectfont,
    anchor=east, text=Gray!70!black, inner sep=0pt},
  redarr/.style={-{Stealth[length=1.35mm]}, line width=0.55pt, BrickRed!70!black},
  greenarr/.style={-{Stealth[length=1.35mm]}, line width=0.55pt, OliveGreen!70!black}
]
\foreach \y/\d in {
  59.1/0,55.85/0,52.6/1,49.35/2,46.1/2,42.85/1,39.6/1,
  36.35/0,33.1/0,29.85/1,26.6/1,23.35/0,20.1/0,16.85/0}
  \node[depth] at (7.35,\y-1.65) {Depth \d};
\begin{scope}[xshift=8.3mm]
\node[ptitle] at (0.4,65.05) {{\bfseries\color{BrickRed!90!black}Shortcut:} copy grandchild answer B};
\node[panel, minimum width=51mm, minimum height=49mm, anchor=north west] at (0,60.75) {};
\begin{scope}[xshift=0.3mm,yshift=-1.65mm]
\node[ann] at (1.65,59.1) {A};
\node[line] at (2.0,59.1) {add(square(add(2,1)),add(2,2))};
\node[line] at (2.0,55.85) {x := <call> square(add(2,1)) </call>};
\node[line] at (5.3,52.6) {x := <call> add(2,1) </call>};
\node[line] at (8.6,49.35) {x := 2  y := 1};
\node[ann] at (8.65,46.1) {B};
\node[hitred] (sB) at (9.05,46.1) {return <return> 3 </return>};
\node[line] at (5.3,42.85) {x := 3};
\node[line] at (5.75,39.6) {return <return> 9 </return>};
\node[ann] at (1.65,36.35) {C};
\node[line] at (2.0,36.35) {x := 9};
\node[line] at (2.0,33.1) {y := <call> add(2,2) </call>};
\node[line] at (5.3,29.85) {x := 2  y := 2};
\node[line] at (5.75,26.6) {return <return> 4 </return>};
\node[ann] at (1.65,23.35) {D};
\node[line] at (2.0,23.35) {y := 4};
\node[hitred] (sOut) at (2.45,20.1) {return <return> 3 </return>};
\draw[redarr] (sB.east) to[out=0,in=35] (sOut.east);
\end{scope}

\begin{scope}[xshift=3.2mm]
\node[ptitle] at (50.4,65.05) {{\bfseries\color{OliveGreen!85!black}Ground Truth:} solve from root A};
\node[psubtitle] at (50.4,62.25) {and subproblem answers C, D};
\node[panel, minimum width=51mm, minimum height=49mm, anchor=north west] at (50,60.75) {};
\begin{scope}[xshift=0.3mm,yshift=-1.65mm]
\node[ann] at (51.65,59.1) {A};
\node[hitgreen] (gexpr) at (52.0,59.1) {add(square(add(2,1)),add(2,2))};
\node[line] at (52.0,55.85) {x := <call> square(add(2,1)) </call>};
\node[line] at (55.3,52.6) {x := <call> add(2,1) </call>};
\node[line] at (58.6,49.35) {x := 2  y := 1};
\node[ann] at (58.65,46.1) {B};
\node[line] at (59.05,46.1) {return <return> 3 </return>};
\node[line] at (55.3,42.85) {x := 3};
\node[line] at (55.75,39.6) {return <return> 9 </return>};
\node[ann] at (51.65,36.35) {C};
\node[hitgreen] (gX) at (52.0,36.35) {x := 9};
\node[line] at (52.0,33.1) {y := <call> add(2,2) </call>};
\node[line] at (55.3,29.85) {x := 2  y := 2};
\node[line] at (55.75,26.6) {return <return> 4 </return>};
\node[ann] at (51.65,23.35) {D};
\node[hitgreen] (gY) at (52.0,23.35) {y := 4};
\node[hitgreen] (gD) at (52.45,20.1) {return <return> 3 </return>};
\draw[greenarr] (gexpr.east) to[out=0,in=112] (gD.east);
\draw[greenarr] (gX.east) to[out=0,in=80] (gD.east);
\draw[greenarr] (gY.east) to[out=0,in=25] (gD.east);
\end{scope}
\end{scope}

\begin{scope}[xshift=6.7mm]
\node[ptitle] at (100.4,65.05) {{\bfseries OOD Failure:} shortcut returns {\color{BrickRed!75!black}3}};
\node[psubtitle] at (100.4,62.25) {correct root return is {\color{OliveGreen!70!black}4}};
\node[panel, minimum width=51mm, minimum height=49mm, anchor=north west] at (100,60.75) {};
\begin{scope}[xshift=0.3mm,yshift=-1.65mm]
\node[ann] at (101.65,59.1) {A};
\node[hitgreen] (oexpr) at (102.0,59.1) {add(square(add(2,1)),add(3,2))};
\node[line] at (102.0,55.85) {x := <call> square(add(2,1)) </call>};
\node[line] at (105.3,52.6) {x := <call> add(2,1) </call>};
\node[line] at (108.6,49.35) {x := 2  y := 1};
\node[ann] at (108.65,46.1) {B};
\node[hitred] (oB) at (109.05,46.1) {return <return> 3 </return>};
\node[line] at (105.3,42.85) {x := 3};
\node[line] at (105.75,39.6) {return <return> 9 </return>};
\node[ann] at (101.65,36.35) {C};
\node[hitgreen] (oX) at (102.0,36.35) {x := 9};
\node[line] at (102.0,33.1) {y := <call> add(3,2) </call>};
\node[line] at (105.3,29.85) {x := 3  y := 2};
\node[line] at (105.75,26.6) {return <return> 5 </return>};
\node[ann] at (101.65,23.35) {D};
\node[hitgreen] (oY) at (102.0,23.35) {y := 5};
\node[hitgreen] (oDtrue) at (102.45,20.1) {return <return> 4 </return>};
\node[hitred] (oDshort) at (102.45,16.85) {return <return> 3 </return>};
\draw[redarr] (oB.east) to[out=0,in=30] (oDshort.east);
\draw[greenarr] (oexpr.east) to[out=0,in=112] (oDtrue.east);
\draw[greenarr] (oX.east) to[out=0,in=80] (oDtrue.east);
\draw[greenarr] (oY.east) to[out=0,in=25] (oDtrue.east);
\end{scope}
\end{scope}
\end{scope}
\end{tikzpicture}}
\end{tcolorbox}
\end{minipage}
\caption{\textbf{Shortcut mechanisms.}
A simpler rule fits the training domain but fails once the training-domain correlation is broken. In recursive computation, reading the complete sequence enables shortcuts that use context irrelevant to the local computation.}
\label{fig:shortcut-examples}
\label{fig:invariance-ood-example}
\label{fig:invariance-shortcut}
\vspace{-0.4em}
\end{figure}
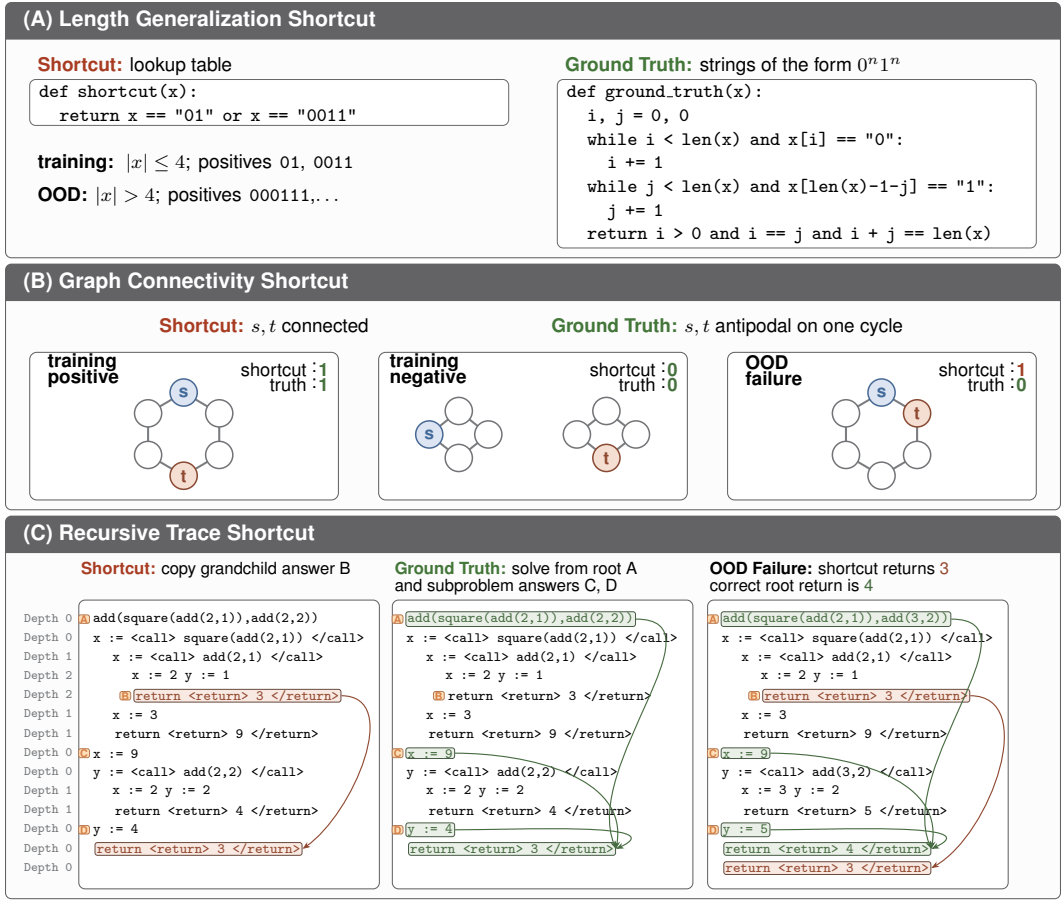

In length generalization, let $f^*(x)=1$ iff $x=0^n1^n$ for some $n\ge 1$. If the training set is $\mathcal{X}_4=\{0,1\}^{\le 4}$, the lookup rule $r_{\rm lookup}(x)=\mathds{1}\{x\in\{01,0011\}\}$ agrees with every training input but rejects every longer balanced string. The intended rule must check the full pattern $0^n1^n$: a nonempty block of $0$s, followed by a block of $1$s of the same length. On this small training set, listing the two positive strings can be shorter than implementing the general rule, so MDL can prefer the lookup even though it fails on the intended test inputs.

For graph classification, suppose positives in the training domain are cycles where marked nodes $s,t$ are antipodal, and negatives are two disjoint cycles with $s,t$ in different components. The shortcut rule ``are $s,t$ connected?'' agrees with all training inputs. The intended rule is more specific: it must check that $s,t$ lie on the same cycle and that their distance is exactly half the cycle length. Connectivity is a coarser and usually simpler predicate, but it fails on connected graphs outside the training domain where $s,t$ are not antipodal. Again, the shortcut is not nonsense; it is a simpler rule that happens to agree with the training inputs.

\subsection{Invariance Under an Observation Function}
\label{sec:recursive-invariance}

The examples fail because the learner can use information that is predictive on the training domain but irrelevant to the intended rule. The same examples also suggest the fix: restrict what the predictor can read. If the predictor only receives the information that the target actually uses, then correlations in the surrounding complete sequence cannot affect its answer. An observation function formalizes this restriction by mapping a complete sequence to the visible context from which the next prediction should be made.

Fix an idempotent observation function $\mask:\Omega\to\Omega$ (so $\mask(\mask(\mathbf{x}))=\mask(\mathbf{x})$) and a local target rule $\frcm$ on visible contexts $\mask(\mathbf{x})$. Viewed as a rule on complete sequences, the same behavior is $\fcot(\mathbf{x}) = \frcm(\mask(\mathbf{x}))$. The simple observation is that any predictor that only reads $\mask(\mathbf{x})$ must give the same answer on sequences with the same visible context. We formalize this as invariance.

\begin{definition}[Invariance]
\label{def:invariance-main}
Let $\mask : \Omega \to \Omega$ be an arbitrary observation function and let $\mathcal{X}\subseteq\Omega$ be the training set. We say a rule $p: \Omega \rightarrow \Sigma$ is \emph{$\mask$-invariant on the training set $\mathcal{X}$} if, for all $\mathbf{x} \in \mathcal{X}$ and $\mathbf{x}_{\mathrm{te}} \in \Omega$, if $\mask(\mathbf{x}) = \mask(\mathbf{x}_{\mathrm{te}})$, then $p(\mathbf{x}) = p(\mathbf{x}_{\mathrm{te}})$.
\end{definition}

\noindent In this definition, $\mathbf{x}$ is a training input, while $\mathbf{x}_{\mathrm{te}}$ is a candidate test input and may lie outside the training domain. Thus invariance is already an out-of-domain statement: it asks whether the prediction stays fixed when the visible context is held fixed but the surrounding complete sequence changes.

Any local rule $q$ used through $\mask$ is $\mask$-invariant by construction, because its complete-sequence prediction has the form $q(\mask(\mathbf{x}))$. For recursive models, $\mask=\maskrec$, and the visible context is the active frame. Different main problems may still require solving the same modular subproblem inside that frame; invariance says that the answer to the subproblem does not depend on which larger problem happened to contain it. We next name the inputs that reuse a training visible context.

\begin{definition}[Equivalent set]
\label{def:equivalent-set-main}
Fix an observation function $\mask:\Omega\to\Omega$. Two complete sequences $\mathbf{x},\mathbf{x}_{\mathrm{te}}\in\Omega$ are \emph{$\mask$-equivalent} if $\mask(\mathbf{x})=\mask(\mathbf{x}_{\mathrm{te}})$. For a training set $\mathcal{X}\subseteq\Omega$, its \emph{$\mask$-equivalent set} is
\begin{align}
\label{eq:equivalent-set}
\operatorname{Eq}_{\mask}(\mathcal{X})
:= \{\mathbf{x}_{\mathrm{te}} \in \Omega : \exists \mathbf{x}\in\mathcal{X}\ \text{such that}\ \mask(\mathbf{x}_{\mathrm{te}})=\mask(\mathbf{x})\}
= \mask^{-1}(\mask(\mathcal{X})).
\end{align}
It contains exactly the complete sequences whose visible context has appeared in $\mathcal{X}$, and in particular $\mathcal{X}\subseteq \operatorname{Eq}_{\mask}(\mathcal{X})\subseteq\Omega$.
\end{definition}

Invariance gives the following transfer guarantee.

\begin{proposition}[OOD Generalization to the Equivalent Set]
\label{prop:invariance-ood-main}
Suppose a predictor $\hat\fcot$ is \emph{consistent with the target $\fcot$ on the training set $\mathcal{X}$} (i.e., $\hat\fcot(\mathbf{x}) = \fcot(\mathbf{x}) = \frcm(\mask(\mathbf{x}))$ for all $\mathbf{x} \in \mathcal{X}$) and \emph{$\mask$-invariant on $\mathcal{X}$}. Then $\hat\fcot$ remains correct throughout the equivalent set: $\hat\fcot(\mathbf{x}_{\mathrm{te}}) = \fcot(\mathbf{x}_{\mathrm{te}}) = \frcm(\mask(\mathbf{x}_{\mathrm{te}}))$ for all $\mathbf{x}_{\mathrm{te}} \in \operatorname{Eq}_{\mask}(\mathcal{X})$.
\end{proposition}

\noindent Thus, for a $\mask$-invariant predictor, covering visible contexts in $\mathcal{X}$ gives correctness on new complete sequences that reuse those contexts. For recursive models, once a subproblem's active frame has appeared in training, changing the parent problem, recursion depth, or other surrounding tokens does not change the predictor's behavior on that subproblem. The guarantee does not cover genuinely new subproblems outside $\mask(\mathcal{X})$.

\subsection{Shortcuts in CoT}
\label{sec:ood-examples}
\label{sec:shortcut-examples}

This also explains why CoT can fail. CoT reads the complete sequence $\mathbf{x}$, so in recursive traces it can attend to tokens outside the active frame $\maskrec(\mathbf{x})$. These tokens can support shortcuts: stack depth, sequence length, repeated tokens, or a value that appears elsewhere in the complete sequence. For example, a value returned by one call may appear outside the active frame used for another prediction. CoT can copy that value on the training inputs instead of computing the answer from the active frame; when the surrounding complete sequence changes, the copy rule can fail. \Cref{fig:shortcut-examples}(C) illustrates this shortcut in recursive computation.

\paragraph{Risks of Shortcuts for OOD Generalization.}
On the same training pairs, compare CoT MDL under $\ell$ with the lifted MDL problem under $\ell^{\mask}$. CoT scores a complete-sequence predictor directly. Under $\ell^{\mask}$, the MDL output is still a complete-sequence predictor, but its length is the length of the shortest visible-context rule used after applying $\mask$.

For a general observation function $\mask$, suppose the target labels are determined by the visible context, $\fcot=\frcm\circ\mask$. Let
\begin{equation}
\label{eq:shortcut-gap-objects}
\hat\fcot\equiv_{\mathcal{X}}\fcot,
\qquad
\hat\fcot_{\mask}:=\mathsf{MDL}_{\ell^{\mask}}(\mathcal{X},\fcot).
\end{equation}
Here $\hat\fcot$ is any finite-length CoT predictor consistent on $\mathcal{X}$, including the CoT MDL output $\mathsf{MDL}_{\ell}(\mathcal{X},\fcot)$ when it is well defined. The subscript $\mask$ records the lifted description language: $\hat\fcot_{\mask}$ is still a complete-sequence predictor. By \Cref{prop:lifted-mdl-local}, for a shortest visible-context rule $\hat\frcm$,
\begin{equation}
\label{eq:lifted-mdl-unpacked-shortcut}
\hat\fcot_{\mask}=\hat\frcm\circ\mask,
\qquad
|\hat\fcot_{\mask}|_{\ell^{\mask}}=|\hat\frcm|_{\ell}.
\end{equation}
Thus the lifted cost of $\hat\fcot_{\mask}$ is the cost of the local rule $\hat\frcm$, not the ordinary cost of the composed predictor $\hat\frcm\circ\mask$. We focus on the strict shortcut gap $|\hat\fcot|_\ell < |\hat\fcot_{\mask}|_{\ell^{\mask}}=|\hat\frcm|_{\ell}$: CoT has found a shorter complete-sequence rule than the shortest visible-context rule whose answer is determined by $\mask(\mathbf{x})$. Theorem~\ref{thm:shortcut-main} shows that this gap forces non-invariance.

\begin{theorem}[Shortcut Forces Non-Invariance]
\label{thm:shortcut-main}
Let $\ell$, $\mathcal{X}\subseteq\Omega$, and an idempotent observation function $\mask:\Omega\to\Omega$ be arbitrary. For any visible-context target rule $\frcm:\Omega\to\Sigma$, set $\fcot:=\frcm\circ\mask$. Let $\hat\fcot_{\mask}:=\mathsf{MDL}_{\ell^{\mask}}(\mathcal{X},\fcot)$ be well defined, and let $\hat\fcot$ be any finite-length predictor satisfying $\hat\fcot\equiv_{\mathcal{X}}\fcot$. If
\begin{equation}
\label{eq:shortcut-gap-main}
|\hat\fcot|_\ell<|\hat\fcot_{\mask}|_{\ell^{\mask}},
\end{equation}
then the CoT predictor $\hat\fcot$ is not $\mask$-invariant on $\mathcal{X}$. That is, there exist $\mathbf{x}\in\mathcal{X}$ and $\mathbf{x}_{\mathrm{te}}\in\mask(\mathcal{X})\subseteq\Omega$ such that
\begin{equation}
\label{eq:shortcut-noninvariance-witness}
\mask(\mathbf{x})=\mask(\mathbf{x}_{\mathrm{te}}), \qquad \hat\fcot(\mathbf{x}_{\mathrm{te}})\neq \hat\fcot(\mathbf{x})=\frcm(\mask(\mathbf{x})).
\end{equation}
\end{theorem}

\begin{proof}
Suppose, toward a contradiction, that $\hat\fcot$ is $\mask$-invariant on $\mathcal{X}$. Then, for every $\mathbf{x}\in\mathcal{X}$, $(\hat\fcot\circ\mask)(\mathbf{x})=\hat\fcot(\mask(\mathbf{x}))=\hat\fcot(\mathbf{x})=\fcot(\mathbf{x})$, so $\hat\fcot\circ\mask$ is feasible for $\mathsf{MDL}_{\ell^{\mask}}(\mathcal{X},\fcot)$. By the definition of the lifted length, the same rule $\hat\fcot$ can be used after applying $\mask$ to represent $\hat\fcot\circ\mask$, hence $|\hat\fcot\circ\mask|_{\ell^{\mask}}\le |\hat\fcot|_\ell$. Since $\hat\fcot_{\mask}$ is the $\ell^{\mask}$-MDL solution,
\begin{equation}
|\hat\fcot_{\mask}|_{\ell^{\mask}}
\le
|\hat\fcot\circ\mask|_{\ell^{\mask}}
\le
|\hat\fcot|_\ell,
\end{equation}
contradicting the shortcut gap. Therefore there exists $\mathbf{x}\in\mathcal{X}$ such that $\hat\fcot(\mask(\mathbf{x}))\neq\hat\fcot(\mathbf{x})$. Set $\mathbf{x}_{\mathrm{te}}:=\mask(\mathbf{x})$. Idempotence gives $\mask(\mathbf{x}_{\mathrm{te}})=\mask(\mathbf{x})$, and consistency gives $\hat\fcot(\mathbf{x})=\fcot(\mathbf{x})=\frcm(\mask(\mathbf{x}))$. Thus $\mathbf{x}_{\mathrm{te}}\in\mask(\mathcal{X})$ is the claimed non-invariance witness.
\end{proof}

Thus a strict shortcut gap certifies non-invariance: the CoT predictor cannot be determined by the visible context alone. The failure is not on an entirely unfamiliar input. For some training sequence $\mathbf{x}$, the CoT predictor is correct on $\mathbf{x}$ but changes its answer on the test input $\mathbf{x}_{\mathrm{te}}:=\mask(\mathbf{x})$. Thus the failure lies in $\mask(\mathcal{X})\subseteq\operatorname{Eq}_{\mask}(\mathcal{X})$: the visible context is familiar, but the surrounding complete sequence has changed.\footnote{The strict gap is sufficient, not necessary. \red{If $\mask(\mathcal{X})\subseteq\mathcal{X}$, then by idempotence any CoT predictor consistent on $\mathcal{X}$ also makes $\hat\fcot\circ\mask$ feasible for the lifted MDL problem with $|\hat\fcot\circ\mask|_{\ell^{\mask}}\le|\hat\fcot|_\ell$, so a strict shortcut gap cannot occur.} Even without such a gap, CoT can still choose a non-invariant rule, for instance through a tie or another consistent choice (\Cref{prop:notb-and-a}).}

Specializing to recursive models, the visible context is the active frame $\maskrec(\mathbf{x})$, a subproblem that already appeared inside a complete sequence in the training set. The learned CoT predictor can answer correctly when the subproblem sits inside that sequence, but fail when the surrounding context changes, including when the same subproblem is presented on its own. \textbf{CoT has learned the subproblem as a pattern tied to the surrounding complete sequence, not as a context-independent local computation that can be reused when the surrounding sequence changes.}

% \begin{tcolorbox}[graybox]
% \textbf{Takeaway.}
% Simulation is not identification: CoT can represent recursive behavior, but recursive learning rules out shortcuts that depend on tokens outside the active frame.
% \end{tcolorbox}

% The experiments test exactly this mechanism. When we test at larger lengths or depths, many active frames remain familiar while the complete sequence changes (\Cref{sec:exp-length-depth}); the case studies create cues visible to CoT that agree with the training answers but are hidden from the active frame and reversed beyond the training domain (\Cref{sec:exp-case-studies}).

%% file: sections/sec5_experiments.tex
\section{Experiments}
\label{sec:experiments}

\red{We use symbolic expression evaluation as a controlled testbed for the theory developed above. The task is simple enough that the intended computation is unambiguous: each expression has an exact modulo-$10$ answer. It is also rich enough to expose the distinctions the theory cares about. Executions decompose into recursive subproblems, and the generator lets us vary the amount of IID training data, the executed trace length, the recursion depth, and whether a shortcut is visible in the full trace. This kind of synthetic program-execution testbed is close to that of \citet{xu2026training}, who train Transformers to execute programs and report evidence of OOD generalization. We use it for a different purpose: to isolate how context visibility affects which rule the learner selects.}

\red{Each execution gives two training views. A chain-of-thought (CoT) learner trains on the complete flattened call/return trace. A recursive-model (RM) learner trains on the same execution one active frame at a time, so each example contains only the subproblem being solved at that moment.}

Expressions are sampled from a small typed library with three kinds of functions. Primitive functions evaluate in one step; composite functions are named definitions that open further calls; and tail-recursive functions are named definitions that can call themselves with updated arguments. For example,
\[
\begin{aligned}
\text{\texttt{sum\_of\_squares}}(a,b)
&\Rightarrow
\text{\texttt{add}}(\text{\texttt{square}}(a), \text{\texttt{square}}(b)),\\
\text{\texttt{accum\_sum}}(n,\mathrm{acc})
&\Rightarrow
\text{\texttt{if\_then\_else}}\bigl(\text{\texttt{less}}(n,1),\, \mathrm{acc},\, \text{\texttt{accum\_sum}}(\text{\texttt{sub}}(n,1),\, \text{\texttt{add}}(\mathrm{acc},n))\bigr).
\end{aligned}
\]
The sampler chooses a root function and recursively fills its arguments with literals or sub-expressions, enforcing type consistency and a bounded depth in the written expression. The bound is on the written expression, not on the executed trace: composite expansion and tail recursion can make a short root problem unfold into a \red{long trace} or a deep recursion stack. The full library and a worked trace example are in \Cref{app:function-library}.

\red{Both learners use the same $6$-layer Transformer backbone and are compared on the same underlying expression pools within each experiment. The primary task metric is final-answer accuracy; the IID experiment also reports per-token test accuracy to diagnose how well each model fits the local next-token objective. Common generator and training details are in \Cref{app:experiment-setup}; experiment-specific setups are in \Cref{app:iid-experiment-setup,app:length-depth-setup,app:shortcut-behavior-setup}.}

\red{The experiments then ask three questions. \textbf{1. IID compositional generalization:} how much IID data is needed to learn new recursive compositions from the same generator? \textbf{2. Length and depth generalization:} what happens when test expressions unfold into longer traces or deeper recursion stacks than the training expressions? \textbf{3. Shortcut behavior:} when CoT fails, do its predictions follow trace-level shortcuts that RM cannot observe?}

\subsection{IID Compositional Generalization}
\label{sec:exp-iid-compositional}

\paragraph{Setup.}
\red{We first remove distribution shift and ask only whether the models learn the recursive computation from IID samples. For each training-set size $N$, RM and CoT train on the same nested set of expressions from the generator and are evaluated on held-out expressions from the same distribution. Each point in \Cref{fig:iid-sample-complexity} asks how well the two learners solve new IID expressions after training on $N$ distinct expressions. Checkpoints are selected by IID validation performance; the full protocol is in \Cref{app:iid-experiment-setup}.}

\paragraph{Results.}
\red{\Cref{fig:iid-sample-complexity} is consistent with the IID prediction of \Cref{sec:id-generalization}: in distribution, there is no large separation between RM and CoT. RM learns the final answer with fewer samples in the middle of the sweep, but CoT reaches the same near-perfect regime once the training set is large enough; on per-token accuracy, CoT is slightly higher at large $N$. The large gaps below appear only after the test distribution changes.}

\begin{figure}[!ht]
\centering
\includegraphics[width=0.86\textwidth]{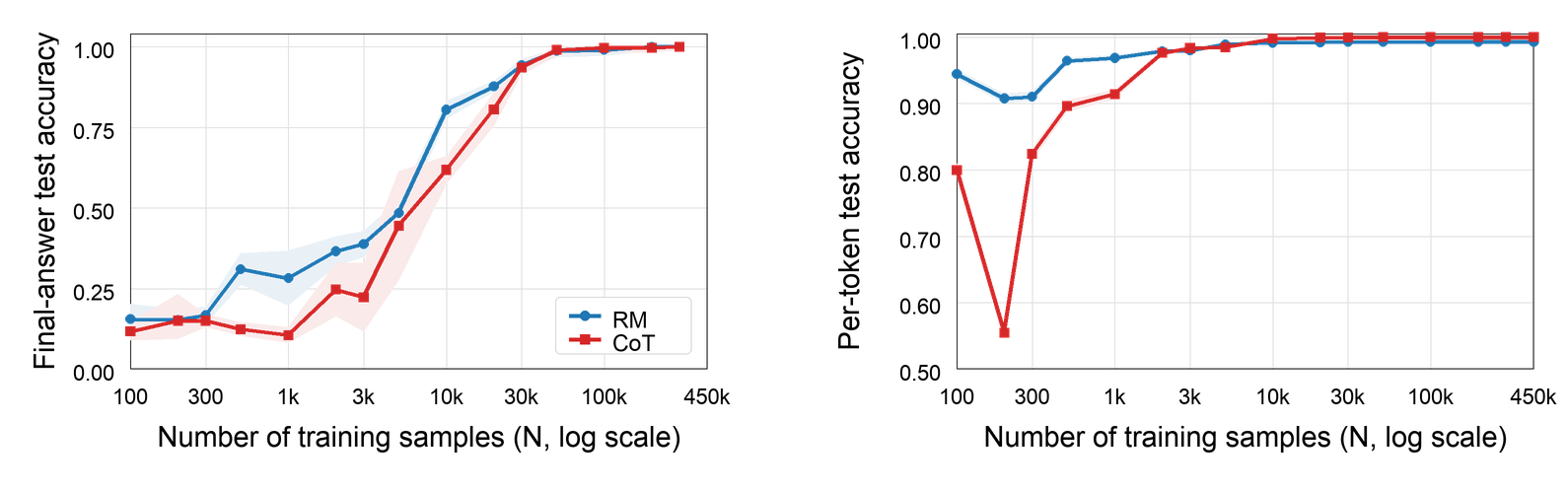}
\caption{\red{IID sample complexity: final-answer and per-token test accuracy.}}
\label{fig:iid-sample-complexity}
\end{figure}

\subsection{Length and Depth Generalization}
\label{sec:exp-length-depth}

\paragraph{Setup.}
We use the same expression generator but change which executions are allowed in training. The length split orders examples by executed trace length: the number of tokens in the executed call/return trace, not the length of the written expression. The depth split orders examples by maximum recursion depth. For a threshold $L$ or $K$, training examples lie below the threshold and OOD examples lie above it; held-out below-threshold examples give the IID check. Details are in \Cref{app:length-depth-setup}.

\paragraph{Results.}
\Cref{tab:per-bin-generalization} separates the below-threshold and above-threshold bins. Below the threshold, both learners are nearly perfect, so the OOD split is not simply a harder version of the IID test. Above the threshold, CoT degrades sharply with longer traces and steadily with deeper stacks. RM remains high across the depth shift and most of the length shift. \red{This matches the invariance picture from \Cref{sec:ood}: many OOD executions still reuse locally familiar subproblems for RM, while the flattened CoT trace has moved outside the training range.} RM drops only in the farthest length bin, where some active frames themselves become longer than those seen in training.

\begin{table}[t]
\centering
\caption{Length/depth generalization. Bins are multiples of the split thresholds $L=595$ and $K=10$.}
\label{tab:per-bin-generalization}
\scriptsize
\setlength{\tabcolsep}{3.2pt}
\renewcommand{\arraystretch}{1.0}

\resizebox{\textwidth}{!}{%
\begin{tabular}{@{}l*{3}{c}!{\hspace{4pt}\vrule width 0.6pt\hspace{4pt}}*{7}{c}@{}}
\toprule
\rowcolor{black!4}
\textbf{Length} & \multicolumn{3}{c}{\textbf{IID}} & \multicolumn{7}{c}{\textbf{OOD}} \\
\cmidrule(lr){2-4}\cmidrule(lr){5-11}
\rowcolor{black!2}
\textbf{Bin ($\times L$)} & $\le0.25$ & $0.25$--$0.5$ & $0.5$--$1.0$ & $1.0$--$1.25$ & $1.25$--$1.5$ & $1.5$--$2.0$ & $2.0$--$2.5$ & $2.5$--$3.0$ & $3.0$--$5.0$ & $>5.0$ \\
\midrule
\rowcolor{RoyalBlue!6}
\textbf{\red{RM acc. (\%)}} & $100.0$ & $100.0$ & $99.6$ & $99.3$ & $99.3$ & $99.6$ & $98.6$ & $95.7$ & $92.4$ & $76.4$ \\
\rowcolor{Red!5}
\textbf{\red{CoT acc. (\%)}} & $100.0$ & $100.0$ & $99.6$ & $79.0$ & $24.5$ & $14.2$ & $10.6$ & $6.2$ & $10.1$ & $9.7$ \\
\rowcolor{black!3}
\textbf{\red{Gap (pt)}} & $+0.0$ & $+0.0$ & $+0.0$ & $+20.2$ & $+74.8$ & $+85.4$ & $+88.0$ & $+89.4$ & $+82.4$ & $+66.7$ \\
\midrule
\midrule
\rowcolor{black!4}
\textbf{Depth} & \multicolumn{3}{c}{\textbf{IID}} & \multicolumn{7}{c}{\textbf{OOD}} \\
\cmidrule(lr){2-4}\cmidrule(lr){5-11}
\rowcolor{black!2}
\textbf{Bin ($\times K$)} & $\le0.3$ & $0.4$--$0.6$ & $0.7$--$1.0$ & $1.1$ & $1.2$--$1.3$ & $1.4$--$1.5$ & $1.6$--$1.7$ & $1.8$--$1.9$ & $2.0$--$2.1$ & $\ge2.2$ \\
\midrule
\rowcolor{RoyalBlue!6}
\textbf{\red{RM acc. (\%)}} & $100.0$ & $99.7$ & $99.6$ & $100.0$ & $99.6$ & $98.6$ & $98.4$ & $99.4$ & $98.6$ & $98.7$ \\
\rowcolor{Red!5}
\textbf{\red{CoT acc. (\%)}} & $100.0$ & $100.0$ & $99.3$ & $68.6$ & $71.6$ & $67.9$ & $63.1$ & $52.2$ & $48.2$ & $38.4$ \\
\rowcolor{black!3}
\textbf{\red{Gap (pt)}} & $+0.0$ & $-0.3$ & $+0.4$ & $+31.4$ & $+28.0$ & $+30.7$ & $+35.3$ & $+47.2$ & $+50.4$ & $+60.4$ \\
\bottomrule
\end{tabular}
}
\end{table}

\subsection{Shortcut Behavior}
\label{sec:exp-case-studies}

\paragraph{Setup and metric.}
The length/depth experiments show large CoT errors, but they do not by themselves identify what CoT learned. We therefore construct train/OOD pairs with a planted shortcut signal $S$. During training, $S$ agrees with the answer, so a predictor can fit the data by using $S$ instead of the intended computation. At OOD evaluation, the task format is unchanged but the coupling between $S$ and the answer is broken. The signal is visible in the flattened CoT trace but hidden from RM's active frame. We report accuracy and shortcut rate, the fraction of OOD examples on which CoT predicts $S$. Details are in \Cref{app:shortcut-behavior-setup}.

\paragraph{Two case studies.}
In the copy-answer case, $S$ is a deep returned value $z$ that CoT can copy from the trace. In the trace-position case, $S$ is whether the final answer token appears at an even or odd trace position. These are different kinds of shortcuts: one is an internal value token, the other is a global property of the trace. Both are present in the flattened trace and absent from RM's active-frame view. Under OOD evaluation, CoT usually predicts $S$; RM remains accurate.

\paragraph{Aggregate results.}
\Cref{tab:robustness} reports all six constructions. \red{The remaining rows vary where the shortcut lives: in a hidden intermediate value, in aggregates of internal returns, or in a value-position statistic in the trace.} CoT is almost perfect IID in every row, so the OOD collapse is not due to an inability to learn these distributions. Once the shortcut is broken, however, CoT's OOD predictions frequently match $S$. RM remains much more stable. These errors are not random extrapolation failures; they track trace-level rules that were valid only in training. \red{This supports the mechanism from \Cref{sec:ood}: CoT can represent the recursive rule, but the flattened trace also exposes simpler non-invariant rules that the active-frame view removes.}

\begin{table}[!htbp]
\centering
\caption{Shortcut behavior across six constructions.}
\label{tab:robustness}
\scriptsize
\setlength{\tabcolsep}{3.2pt}
\renewcommand{\arraystretch}{1.0}
\resizebox{\textwidth}{!}{%
\begin{tabular}{@{}l>{\columncolor{RoyalBlue!6}}c>{\columncolor{Red!5}}c!{\hspace{4pt}\vrule width 0.6pt\hspace{4pt}}>{\columncolor{RoyalBlue!6}}c>{\columncolor{Red!5}}ccc@{}}
\toprule
\rowcolor{black!4}
\textbf{Construction} & \multicolumn{2}{c}{\textbf{IID}} & \multicolumn{4}{c}{\textbf{OOD}} \\
\cmidrule(lr){2-3}\cmidrule(lr){4-7}
\rowcolor{black!2}
& \textbf{RM acc.} & \textbf{CoT acc.} & \textbf{RM acc.} & \textbf{CoT acc.} & \textbf{Gap (pt)} & \textbf{Shortcut rate} \\
\midrule
Copy answer from deep return (Case 1) & $100.0\%$ & $100.0\%$ & $87.3\%$ & $\phantom{0}7.0\%$ & $+80.3$ & $92.7\%$ \\
Hidden intermediate literal & $\phantom{0}90.3\%$ & $\phantom{0}99.7\%$ & $86.0\%$ & $\phantom{0}0.3\%$ & $+85.7$ & $99.3\%$ \\
Long-range aggregate & $\phantom{0}99.7\%$ & $100.0\%$ & $95.7\%$ & $10.3\%$ & $+85.3$ & $88.0\%$ \\
Near-answer aggregate & $100.0\%$ & $100.0\%$ & $99.7\%$ & $\phantom{0}1.0\%$ & $+98.7$ & $79.7\%$ \\
Position-weighted aggregate & $100.0\%$ & $100.0\%$ & $97.0\%$ & $14.7\%$ & $+82.3$ & $69.7\%$ \\
Trace-position parity (Case 2) & $100.0\%$ & $100.0\%$ & $99.7\%$ & $\phantom{0}9.7\%$ & $+90.0$ & $90.3\%$ \\
\bottomrule
\end{tabular}
}
\end{table}

%% file: sections/appendix_related.tex
\section{Related Work}
\label{sec:related-work}

Our work is closest to recursive and context-isolated approaches to language-model reasoning.
\citet{zhang2025recursive} introduce \emph{recursive language models}, where an LLM treats a long prompt as an external environment and recursively inspects or decomposes selected snippets; related systems use recursive spawning, divide-and-conquer reasoning, context folding, adaptive decomposition, or planner-executor structure to extend long-context reasoning and planning~\citep{yang2025recursive,lee2023recursion,prasad2024adapt,schroeder2025thread,pan2025learning,sun2025scalinglonghorizonllmagent,zhang2025recap}.
These works motivate the same design principle that the model should sometimes see a fresh subproblem sequence rather than the entire complete sequence.
They mainly ask how recursion or context management improves inference, computation, or task accuracy.
We ask a different statistical question: after setting aside the context-length advantage, does context isolation change what is learned?
This places the paper between several lines of work.
Chain-of-thought, scratchpads, least-to-most prompting, decomposed prompting, and tree-of-thought expose or organize intermediate computation in a visible sequence~\citep{nye2021show,wei2022chain,zhou2022least,khot2022decomposed,yao2024tree}.
Retrieval, memory, summarization, and agent/tool systems choose which information is shown to the model at each step~\citep{yang2025pencil,packer2024memgptllmsoperatingsystems,wang2024augmenting,wu2025resumunlockinglonghorizonsearch,yu2025memagent,yao2023react,schick2023toolformer,hong2024metagpt,wu2023autogen}.
Classical Occam/MDL arguments and the bitter-lesson view explain why a broader CoT learner need not pay much for IID prediction: a general function approximator can still contain the right recursive rule compactly~\citep{rissanen1978modeling,grunwald2007minimum,sutton2019bitter}.
Shortcut-learning, simplicity-bias, and spurious-correlation work point in the opposite direction under distribution shift: information that is harmless in distribution can support easier or shorter non-invariant explanations that fail off domain~\citep{geirhos2020shortcut,shah2020pitfalls,pezeshki2021gradient,sagawa2020investigation,nagarajan2021understanding,shi2023distracted,liu2024lost,peters2016causal,arjovsky2019invariant}.
Our contribution is to put these views together.
CoT can simulate recursive computation and therefore loses little in IID sample complexity, but the complete sequence also exposes shorter non-invariant explanations; recursive context isolation helps OOD generalization by removing those explanations from the learner's view.

%% file: sections/appendix_compression.tex
\section{Background: Prefix Codes and the Realizable MDL Bound}
\label{sec:appendix-compression}

This appendix provides the technical background for the MDL framework used in \Cref{sec:id-learning}. \Cref{sec:prefix-codes} explains how prefix-code lengths can be read as probability weights, which is the formal basis for interpreting a complexity measure as a prior. \Cref{sec:proof-mdl-gen} gives a self-contained proof of \Cref{thm:mdl-gen-bound}.

\subsection{Prefix Codes and the Code--Distribution Correspondence}
\label{sec:prefix-codes}

Let $\mathcal{U}$ be a countable set. A \emph{code} for $\mathcal{U}$ is a mapping $C: \mathcal{U} \to \{0,1\}^*$ that assigns each $u \in \mathcal{U}$ a binary codeword $C(u)$; its \emph{codelength function} is $L_C(u) = |C(u)|$. A code is \emph{prefix-free} if no codeword is a prefix of any other, so that a concatenation $C(u_1)C(u_2)\cdots C(u_n)$ is uniquely decodable from left to right~\citep{cover2006elements}. The Kraft inequality characterizes which codelength assignments are realizable by prefix-free codes.

\begin{theorem}[Kraft Inequality]
\label{thm:kraft}
For nonnegative integer lengths $L(u_1), L(u_2), \ldots$, a prefix-free code over $\mathcal{U}$ with these codelengths exists if and only if
\begin{align}
\sum_{u \in \mathcal{U}} 2^{-L(u)} \leq 1.
\end{align}
\end{theorem}

\paragraph{Code lengths as probability weights.}
The Kraft inequality lets prefix-code lengths be read as (sub-)probability weights. Given a prefix code $C$ with lengths $L_C(u)$,
\begin{align}
\label{eq:code-to-prob}
P(u) = 2^{-L_C(u)}
\end{align}
is a (sub-)probability distribution on $\mathcal{U}$: short codewords correspond to high probability. Conversely, any distribution $P$ on $\mathcal{U}$ gives valid prefix-code lengths
\begin{align}
\label{eq:prob-to-code}
L_C(u) = \lceil -\log_2 P(u) \rceil,
\end{align}
since $\sum_u 2^{-\lceil -\log_2 P(u)\rceil} \leq \sum_u P(u) = 1$. The ceiling contributes at most one bit of overhead.

\begin{proposition}[Code Lengths and Distributions]
\label{prop:code-distribution}
Prefix-code lengths over $\mathcal{U}$ and (sub-)probability weights on $\mathcal{U}$ are interchangeable up to integer rounding~\citep[Theorems~5.2.1, 5.3.1]{cover2006elements}: lengths $L(u)$ satisfying Kraft define weights $2^{-L(u)}$, and any distribution $P$ gives valid lengths $\lceil -\log_2 P(u) \rceil$. Assigning probability $P(u)$ to an outcome is therefore equivalent to encoding it in $-\log_2 P(u)$ bits, up to rounding.
\end{proposition}

In particular, any complexity measure $\ell$ on predictors induces a (sub-)prior $2^{-\ell(h)}$, under which shorter descriptions receive higher prior weight. This is the reading of $\ell$ as a prior used in the main text.

\subsection{Proof of the Realizable MDL Bound}
\label{sec:proof-mdl-gen}

We restate \Cref{thm:mdl-gen-bound} and give a self-contained proof following the standard Occam argument~\citep[Section~7.3]{shalev2014understanding}.

\noindent\textbf{\Cref{thm:mdl-gen-bound}} (Realizable MDL Generalization Bound, restated)\textbf{.}
\textit{\red{Fix a complexity measure $\ell$ satisfying the Kraft inequality $\sum_h 2^{-|h|_\ell} \leq 1$, a target $f^*$ with $|f^*|_\ell<+\infty$, and $\delta\in(0,1)$.} Let $\mathcal{X} = \{\mathbf{x}^{(1)}, \ldots, \mathbf{x}^{(n)}\}$ be $n$ i.i.d.\ samples from a distribution $\mathcal{D}$ on $\Omega$, and write $\mathcal{L}_{\mathcal{D}}(h) = \Pr_{\mathbf{x}\sim\mathcal{D}}[h(\mathbf{x}) \neq f^*(\mathbf{x})]$ and $\widehat{\mathcal{L}}(h) = \frac{1}{n}\sum_{i=1}^n \mathds{1}(h(\mathbf{x}^{(i)}) \neq f^*(\mathbf{x}^{(i)}))$. Then with probability at least $1 - \delta$ over the draw of $\mathcal{X}$, \red{every finite-length predictor $h$} with $\widehat{\mathcal{L}}(h)=0$ satisfies}
\begin{align}
\mathcal{L}_{\mathcal{D}}(h) \leq \frac{|h|_\ell\ln 2 + \ln(1/\delta)}{n},
\end{align}
\textit{In particular, the MDL learner \red{$\hat f = \mathsf{MDL}_\ell(\mathcal{X}, f^*)$} satisfies $\widehat{\mathcal{L}}(\hat f) = 0$ and $|\hat f|_\ell\le |f^*|_\ell$, so}
\begin{align}
\mathcal{L}_{\mathcal{D}}(\hat f) \leq \frac{|\hat f|_\ell\ln 2 + \ln(1/\delta)}{n}
\leq \frac{|f^*|_\ell\ln 2 + \ln(1/\delta)}{n}.
\end{align}

\noindent The proof uses the realizable Occam argument: a predictor with large population error is unlikely to fit all training labels, and the Kraft inequality lets us union bound over predictors.

\paragraph{Step 1: A single consistent predictor.}
Fix any predictor $h$, and write $p=\mathcal{L}_{\mathcal{D}}(h)$. The probability that $h$ agrees with $f^*$ on all $n$ samples is
\begin{align}
\Pr[\widehat{\mathcal{L}}(h)=0]=(1-p)^n \leq \exp(-np).
\label{eq:consistent-single}
\end{align}
Therefore, for any threshold $\epsilon>0$,
\begin{align}
\Pr\!\left[\widehat{\mathcal{L}}(h)=0 \text{ and } \mathcal{L}_{\mathcal{D}}(h)>\epsilon\right]
\leq \exp(-n\epsilon).
\end{align}

\paragraph{Step 2: Weighted union bound via the Kraft inequality.}
Set
\begin{align}
\epsilon_h=\frac{|h|_\ell\ln 2+\ln(1/\delta)}{n}.
\end{align}
By \Cref{eq:consistent-single},
\begin{align}
\Pr\!\left[\widehat{\mathcal{L}}(h)=0 \text{ and } \mathcal{L}_{\mathcal{D}}(h)>\epsilon_h\right]
\leq \exp(-n\epsilon_h)
=\delta\,2^{-|h|_\ell}.
\end{align}
Taking a union bound over \red{all finite-length predictors $h$} and using Kraft,
\begin{align}
\sum_h \delta\,2^{-|h|_\ell} \leq \delta.
\end{align}
Thus, with probability at least $1-\delta$, no predictor is both consistent with the sample and has population loss above its threshold $\epsilon_h$. Equivalently, every $h$ with $\widehat{\mathcal{L}}(h)=0$ satisfies
\begin{align}
\mathcal{L}_{\mathcal{D}}(h) \leq \frac{|h|_\ell\ln 2+\ln(1/\delta)}{n}.
\end{align}

\paragraph{MDL corollary.} The MDL learner \red{$\hat f = \mathsf{MDL}_\ell(\mathcal{X}, f^*)$} agrees with $f^*$ on every input in $\mathcal{X}$, so $\widehat{\mathcal{L}}(\hat f) = 0$. Moreover, by definition $|\hat f|_\ell \leq |f^*|_\ell$. Substituting:
\begin{align}
\mathcal{L}_{\mathcal{D}}(\hat f)
\leq \frac{|\hat f|_\ell\ln 2 + \ln(1/\delta)}{n}
\leq \frac{|f^*|_\ell\ln 2 + \ln(1/\delta)}{n}.
\end{align}
This completes the proof. \qed

%% file: sections/appendix_transformer.tex
\section{Transformer Architecture}
\label{app:transformer}
\label{sec:appendix-transformer}

This appendix provides a self-contained formal definition of the decoder-only Transformer architecture used throughout the paper. Our formulation follows the standard architecture of \citet{vaswani2017attention}, with hard attention as the default attention mechanism, following recent theoretical work~\citep{hahn2020theoretical, merrill2022saturated}. We assume exact arithmetic over the reals unless stated otherwise.

Let $\Sigma$ be a finite vocabulary with $|\Sigma| = K$, let $\Omega:=\Sigma^*$, and write $\concat$ for concatenation. Token sequences are written as $\mathbf{x} = (x_1, x_2, \ldots, x_T) \in \Sigma^T$, where $T$ is the sequence length. Embedding sequences are written as $H \in \mathbb{R}^{T \times d}$, where $d$ is the embedding dimension. We use $H_i \in \mathbb{R}^d$ to denote the embedding at position $i$.

\paragraph{Transformer layers.}
A Transformer of depth $L$ transforms an input embedding sequence $H^{(0)} = H \in \mathbb{R}^{T \times d}$ through $L$ layers. We allow layer-dependent hidden dimensions $d^{(0)}, d^{(1)}, \ldots, d^{(L)}$ with $d^{(0)} = d$, so $H^{(\ell)} \in \mathbb{R}^{T \times d^{(\ell)}}$. Each layer applies two sublayers: an attention mechanism followed by a feed-forward network:
\begin{align}
\tilde{H}^{(\ell)} = \mathsf{Attn}^{(\ell)}(H^{(\ell-1)}), \qquad H^{(\ell)} = \mathsf{FFN}^{(\ell)}(\tilde{H}^{(\ell)}).
\end{align}
Our results hold regardless of whether residual connections are included.

\paragraph{Hard attention.}
The attention mechanism enables data-dependent selection of information across positions. For each attention head, queries, keys, and values are computed via affine transformations:
\begin{align}
Q_i = W_Q H_i + b_Q, \qquad K_i = W_K H_i + b_K, \qquad V_i = W_V H_i + b_V,
\end{align}
where $W_Q, W_K, W_V$ are projection matrices and $b_Q, b_K, b_V$ are bias vectors (suppressing the layer index $\ell$). Each head may also include an output affine map $W_O,b_O$. We use \emph{causal masking}: position $i$ can only attend to positions $j \le i$. We analyze \emph{hard attention}~\citep{hahn2020theoretical, merrill2022saturated}, the low-temperature limit of softmax attention where weights concentrate on score-maximizing positions. Let
\begin{equation}
M_i = \bigl\{j \le i : \langle Q_i, K_j \rangle = \max_{k \le i} \langle Q_i, K_k \rangle\bigr\}
\end{equation}
denote the set of causally-visible positions achieving the maximum score, and define attention weights $\alpha_{i,j} = \mathbf{1}[j \in M_i] / |M_i|$. The pre-output attention value is $U_i = \sum_{j} \alpha_{i,j}\, V_j$, and the head output is $W_O U_i + b_O$ when an output map is present. The generalization bound of \Cref{thm:iid-sample-complexity} does not depend on the tie-breaking rule; the simulation construction of \Cref{thm:cot-simulation} (see \Cref{app:proof-cot-simulation}) specifically instantiates the \emph{rightmost-hard} variant~\citep{hahn2020theoretical}, in which ties are resolved in favor of the rightmost position.

\paragraph{Multi-head attention.}
A multi-head attention layer at layer $\ell$ with $h^{(\ell)}$ heads computes $h^{(\ell)}$ attention outputs in parallel. Each head $r \in \{1, \ldots, h^{(\ell)}\}$ has its own learned block of affine maps $W_Q^{(\ell,r)}, W_K^{(\ell,r)}, W_V^{(\ell,r)}, W_O^{(\ell,r)}$ and the corresponding biases, and applies hard attention as defined above to produce $\mathsf{Attn}^{(\ell,r)}(H)$. The outputs are summed:
\begin{equation}
\mathsf{Attn}^{(\ell)}(H) = \sum_{r=1}^{h^{(\ell)}} \mathsf{Attn}^{(\ell,r)}(H).
\end{equation}
In practice, Transformers partition the input dimension across heads, concatenate their outputs, and apply a linear projection; the summation formulation is an equivalent notational simplification for the arguments below.

\paragraph{Feed-forward network.}
The $\mathsf{FFN}$ sublayer at layer $\ell$ is a position-wise MLP with ReLU activations, mapping $\mathbb{R}^{d^{(\ell-1)}} \to \mathbb{R}^{d^{(\ell)}}$. It has internal depth $D^{(\ell)}$ and hidden dimensions specified by a sequence $\bar{d}^{(\ell)} \in \mathbb{N}^{D^{(\ell)}+1}$ where the first element equals $d^{(\ell-1)}$ and the last equals $d^{(\ell)}$. The results extend to any piecewise-linear activation~\citep{bartlett1998almost}.

\paragraph{Product gates (for the simulation theorem).}
The simulation construction of \Cref{thm:cot-simulation} (\Cref{app:proof-cot-simulation}) additionally assumes the FFN class admits coordinatewise product gates, equivalently ReGLU-style gates: a constant-depth position-wise network can compute a bilinear form $\langle u, v \rangle$ between two hidden vectors $u, v \in \mathbb{R}^{d^{(\ell)}}$. This assumption does not affect the IID comparison of \Cref{thm:iid-sample-complexity}, which depends on the description length of the predictor returned by MDL.

\paragraph{Input layer.}
The Transformer operates on token sequences $\mathbf{x} \in \Sigma^T$. The \emph{token embedding} $\mathsf{TE} : \Sigma \to \mathbb{R}^d$, parametrized by $W_{\mathsf{TE}} \in \mathbb{R}^{d \times K}$, maps each token to a $d$-dimensional vector. The \emph{positional embedding} $\mathsf{PE} : \{1, \ldots, T\} \to \mathbb{R}^d$ encodes position information. The input embedding sequence is:
\begin{align}
H = \bigl(\mathsf{TE}(x_1) + \mathsf{PE}(1),\; \ldots,\; \mathsf{TE}(x_T) + \mathsf{PE}(T)\bigr) \in \mathbb{R}^{T \times d}.
\end{align}
Our architecture allows any fixed (non-learnable) positional embedding $\mathsf{PE}$.

\paragraph{Output layer.}
The output is extracted from the final position $H^{(L)}_T \in \mathbb{R}^{d^{(L)}}$ of the last layer. A decoding matrix $W_{\mathsf{DE}} \in \mathbb{R}^{K \times d^{(L)}}$ maps the final representation to logits over the vocabulary:
\begin{align}
f_\theta(\mathbf{x}) = \arg\max_{a \in \Sigma}\, \bigl[W_{\mathsf{DE}}\, H^{(L)}_T\bigr]_a \in \Sigma.
\end{align}
In practice, $W_{\mathsf{DE}}$ may share weights with the token embedding $W_{\mathsf{TE}}$ (\emph{weight tying}); the analysis permits both tied and untied configurations.

\paragraph{Hypothesis classes.}
A Transformer architecture $\mathcal{A}$ is fully specified by:
\begin{align}
\mathcal{A} = \bigl\{K,\; L,\; \{d^{(\ell)}\}_{\ell=0}^L,\; \{h^{(\ell)}\}_{\ell=1}^L,\; \{\bar{d}^{(\ell)}\}_{\ell=1}^L,\; \mathsf{PE}\bigr\},
\end{align}
encoding vocabulary size $K$, depth $L$, layer-wise hidden dimensions, attention heads per layer, $\mathsf{FFN}$ internal dimensions, and positional encoding. We write $L(\mathcal{A}) := L$ for the depth and $W(\mathcal{A})$ for the number of independently stored learned parameters. Standard reuse of the same layer across sequence positions and autoregressive time steps is counted once. In lifted architectures, lane-specific or subroutine-specific copies of a learned block are counted separately, while fixed wiring and fixed non-learned operations are not counted. Equivalently, in the main text we write $L_\mathcal{A}:=L(\mathcal{A})$ and $W_\mathcal{A}:=W(\mathcal{A})$. For the lifted constructions below, the architecture specification may also include several named hidden-state lanes and fixed non-learned operations, as described in \Cref{app:lane-wise-architecture}. For a fixed architecture $\mathcal{A}$, the hypothesis class of next-token generators is
\begin{equation}
\mathcal{F}_{\mathcal{A}} := \{f_\theta : \theta \in \mathbb{R}^{W(\mathcal{A})}\}.
\end{equation}
We say that a function $f : \Sigma^* \to \Sigma$ is \emph{realized} by architecture $\mathcal{A}$ if $f \in \mathcal{F}_{\mathcal{A}}$, i.e., there exist parameters $\theta \in \mathbb{R}^{W(\mathcal{A})}$ such that $f_\theta(\mathbf{x}) = f(\mathbf{x})$ for all $\mathbf{x} \in \Sigma^*$.

\subsection{Lanes and Unfolded Parameter Count}
\label{app:lane-wise-architecture}

\paragraph{Learned blocks and parameter count.}
A Transformer architecture specifies a finite list of learned scalar parameters
\begin{equation}
\theta=(\theta_1,\ldots,\theta_{W(\mathcal{A})})\in\mathbb{R}^{W(\mathcal{A})}.
\end{equation}
All learned matrices and biases appearing in the computation graph are assembled from this list. We count $W(\mathcal{A})$ as the number of independent learned scalars in this list. Fixed positional features, fixed masks, fixed exact-match selectors, fixed copy operations, fixed merge gates, and other non-learned arithmetic operations are part of the architecture specification and are not counted in $W(\mathcal{A})$.

\paragraph{Lanes.}
A lane is just a named copy of the hidden-state stream, used to store a different simulated view of the same input sequence. A layer may maintain several such lanes
\begin{equation}
\begin{aligned}
H^{(\ell,r)}=(H^{(\ell,r)}_1,\ldots,H^{(\ell,r)}_T),
\qquad r\in\mathcal{R}_\ell,
\end{aligned}
\end{equation}
where each lane has a prescribed width.

\paragraph{Unfolded learned-parameter count.}
In lifted architectures we use the dense unfolded convention. If a construction applies an original learned block to several lanes or subroutines, these applications are represented by independent learned copies and all copies are counted in $W(\mathcal{A})$. In a simulation proof, we may choose these copied parameters to have the same numerical values as the original block. This equality is part of the witness parameter setting, not an architectural weight-sharing constraint.

Fixed positional features, fixed masks, fixed exact-match selectors, fixed copy operations, fixed merge gates, and other non-learned arithmetic operations are part of the architecture specification and are not counted in $W(\mathcal{A})$.

\paragraph{Autoregressive generation.}
The Transformer serves as a next-token generator $f_\theta : \Omega \to \Sigma$ (\Cref{sec:modularity-lm}). Given an initial input sequence $\mathbf{x}_0 \in \Omega$, autoregressive generation repeatedly applies $f_\theta$ and appends the output:
\begin{align}
\mathbf{x}_{t+1} = \mathbf{x}_t \concat f_\theta(\mathbf{x}_t), \qquad t\ge 0,
\end{align}
as in \Cref{sec:ar-setup}. At each step, the same parameters $\theta$ are used. In the CoT setting, $f_\theta$ reads the complete sequence $\mathbf{x}_t$. In the recursive model setting (\Cref{sec:recursive-model}), the predictor reads the active frame $\maskrec(\mathbf{x}_t)$.

\paragraph{Teacher-forced training.}
In practice, autoregressive models are trained using teacher forcing~\citep{williams1989learning}: at each step, the model predicts the next token conditioned on the ground-truth prefix rather than its own previous predictions. Given a dataset $S$ of prefix--next-token pairs $(\mathbf{x}, y)$ drawn from correct generated sequences, the training objective is the empirical cross-entropy loss:
\begin{align}
\widehat{\mathcal{L}}(\theta;\, S) = \frac{1}{|S|} \sum_{(\mathbf{x},\, y) \in S} -\log\, p_\theta(y \mid \mathbf{x}),
\end{align}
where $p_\theta(a \mid \mathbf{x}) = \mathrm{softmax}(W_{\mathsf{DE}}\, H^{(L)}_T)_a$ is the predicted probability of token $a$ given prefix $\mathbf{x}$.

%% file: sections/appendix_simulation.tex
\section[Proof of CoT Simulation]{Proof of \Cref{thm:cot-simulation} (CoT Simulation)}
\label{app:proof-cot-simulation}
\label{sec:appendix-simulation}

\paragraph{Proof idea.}
When we replay the recursive stack update, an ordinary token inside an unfinished call/return block has two possible roles. Before the block closes, it is part of the current unfinished suffix. After the block closes, the same token belongs to the frame that eventually owns it. We therefore maintain two simulated hidden states at every position: an \emph{owner state} $H^O$ for the eventual frame, and a \emph{pending state} $H^P$ for the unfinished suffix. The proof has three steps. First, we characterize $\maskrec(\mathbf{x}_{1:i})$ as either an owner sequence $O_j$ or a pending sequence $P_i$ (\Cref{app:imagined-inputs}). Second, we run each layer of $\mathcal{A}$ on both states in parallel (\Cref{app:lifted-sim}). Third, we read out the state corresponding to the current active frame and count resources (\Cref{app:readout}).

\paragraph{The lifted invariant.}
Everything below is aimed at maintaining, through every layer $\ell \in \{0, \ldots, L\}$, the invariant
\begin{equation}
\label{eq:lifted-invariant}
\begin{aligned}
H_j^{O,(\ell)}
&= H^{\mathcal{A},\ell}_{|O_j|}(O_j)
&&\text{for every } x_j \in \Sigma_{\mathrm{ord}},\\
H_j^{P,(\ell)}
&= H^{\mathcal{A},\ell}_{|P_j|}(P_j)
&&\text{for every } j \in \mathcal{P}.
\end{aligned}
\end{equation}
where $O_j, P_j$ are the two imagined input sequences constructed in \Cref{app:imagined-inputs} and $\mathcal{P}$ is the set of pending positions.

\paragraph{Assumptions.}
We work on legal prefixes (\Cref{def:legal-prefix}); non-legal inputs are mapped arbitrarily by $\mathcal{A}'$. The lifted family is rightmost-hard with product-gated FFNs and reindexable positional features (\Cref{app:transformer}). Here rightmost-hard only fixes the tie-breaking rule, and reindexable positional features mean that the construction can use the positional feature for a simulated position such as $|O_j|$ or $|P_j|$. Without loss of generality every input $\mathbf{x}$ begins with a distinguished sentinel $\BOS \in \Sigma_{\mathrm{ord}}$, and we adopt the convention $\frcm(\epsilon) := \frcm(\BOS)$ for edge cases in which the observation is empty. We use the unfolded parameter-count convention of \Cref{app:lane-wise-architecture}: lane-specific and subroutine-specific applications of an original learned block are implemented as independent learned copies in $\mathcal{A}'$, and the simulator parameter setting assigns those copies the same numerical values as the corresponding block of $\mathcal{A}$. For constant-depth claims we use known constant-depth constructions from \citet{yang2025pencil} (prefix sums, conditional maxima, rightmost exact-match copying, local piecewise-linear computations, $O(\log n)$-precision integer multiplication); together with product gates these suffice to compute bilinear forms over $O(d^{(\ell)})$-dimensional hidden vectors.

\paragraph{Restricted-attention scope.}
The finite-penalty implementation of restricted attention below is stated for the bounded-precision parameter classes used in \Cref{sec:id-generalization}; with exact non-learned hard masks for the computed allowed-position sets, the same simulation applies without this bounded-score step.

\paragraph{Proof convention.}
The construction uses the unfolded convention of \Cref{app:lane-wise-architecture}. Each lane-specific or subroutine-specific application of an original learned attention head, FFN, or decoder block is an independent learned copy in $\mathcal{A}'$. To prove simulation, we exhibit a parameter setting in which each copy has the same numerical value as the corresponding block of $\mathcal{A}$. For readability, the formulas below write these copied matrices using the original symbols; this is a notational shorthand and not an architectural weight-sharing assumption. The bookkeeping block, allowed-position sets, exact-match selectors, copy operations, and merge gates remain fixed non-learned parts of the architecture.

\subsection[What the Active-Frame Map Looks Like]{What $\maskrec(\mathbf{x})$ Looks Like}
\label{app:imagined-inputs}

The stack replay can view the same ordinary token in two ways. Before the surrounding block closes, the token contributes to the current unfinished suffix; after the close, it contributes to its owner frame. We encode these two views by two imagined input sequences $O_j$ and $P_j$. These are not extra generated sequences; they are the sequences that the original architecture $\mathcal{A}$ would see if it were run on the relevant frame.

Recall from the recursive-model setup in \Cref{sec:recursive-model} that $\Sigma_{\mathrm{ctrl}} = \{\CALL, \UNCALL, \ANSWER, \UNANSWER\}$ and $\Sigma_{\mathrm{ord}} = \Sigma \setminus \Sigma_{\mathrm{ctrl}}$. The replay fires on completed suffix blocks: $\UNCALL$ triggers a push when the top frame ends with $\CALL\,\mathbf{q}\,\UNCALL$ ($\mathbf{q} \in \Sigma_{\mathrm{ord}}^*$), and $\UNANSWER$ triggers a pop when it ends with $\ANSWER\,\mathbf{a}\,\UNANSWER$ ($\mathbf{a} \in \Sigma_{\mathrm{ord}}^*$).

We only need to consider flattened sequences that the recursive stack procedure could actually produce. The following condition makes that precise.

\begin{definition}[Legal prefix]
\label{def:legal-prefix}
A prefix $\mathbf{x}_{1:i} \in \Sigma^*$ is \emph{legal} if, during the replay of the procedure in \Cref{fig:obs-function}(C): (i) every $\UNCALL$ closes a matching $\CALL\,\mathbf{q}\,\UNCALL$ block with $\mathbf{q} \in \Sigma_{\mathrm{ord}}^*$ and fires the push rule; (ii) every $\UNANSWER$ closes a matching $\ANSWER\,\mathbf{a}\,\UNANSWER$ block with $\mathbf{a} \in \Sigma_{\mathrm{ord}}^*$ and fires the pop rule; (iii) the replay never pops below the root. Equivalently, $\mathbf{x}_{1:i}$ is a prefix of a well-formed recursive generated sequence with control-free payloads.
\end{definition}

\paragraph{Bookkeeping labels.}
The following labels only keep track of which frame each token belongs to and whether the current top frame has an unfinished call or return block.
For each position $i$, let
\begin{equation}
\begin{aligned}
\lambda_i &:= \max\bigl(\{t < i : x_t \in \Sigma_{\mathrm{ctrl}}\} \cup \{0\}\bigr), \\
\beta_i &:= \max\bigl(\{t \le i : x_t \in \Sigma_{\mathrm{ctrl}}\} \cup \{0\}\bigr),
\end{aligned}
\end{equation}
and the completed depth $c_i := \sum_{t \le i} (\mathbf{1}[x_t = \UNCALL] - \mathbf{1}[x_t = \UNANSWER])$. The current committed frame id and its parent are
\begin{equation}
\begin{aligned}
F_i &:= \max(\{\lambda_t : t \le i,\; x_t = \UNCALL,\; c_t = c_i\} \cup \{0\}), \\
\mathrm{par}_i &:= \max(\{\lambda_t : t \le i,\; x_t = \UNCALL,\; c_t = c_i - 1\} \cup \{0\}).
\end{aligned}
\end{equation}
The pending-block type is
\begin{equation}
m_i := \begin{cases}
\mathsf{call}, & \beta_i > 0 \text{ and } x_{\beta_i} = \CALL, \\
\mathsf{return}, & \beta_i > 0 \text{ and } x_{\beta_i} = \ANSWER, \\
\bot, & \text{otherwise}.
\end{cases}
\end{equation}

\begin{lemma}[Replay normal form]
\label{lem:replay-normal-form}
For every legal prefix $\mathbf{x}_{1:i}$:
\begin{enumerate}[leftmargin=2em, itemsep=2pt, topsep=2pt]
\item Every non-top frame of the stack is a string in $\Sigma_{\mathrm{ord}}^*$.
\item The top frame has exactly one of the three forms
$\mathbf{U}$, $\mathbf{U} \concat \CALL \concat \mathbf{q}$, $\mathbf{U} \concat \ANSWER \concat \mathbf{a}$ with $\mathbf{U}, \mathbf{q}, \mathbf{a} \in \Sigma_{\mathrm{ord}}^*$.
\item In particular, the top frame contains at most one unfinished block; if one exists it is the suffix starting at $x_{\beta_i}$.
\item $m_i \neq \bot$ iff $\maskrec(\mathbf{x}_{1:i})$ ends with the unfinished suffix $x_{\beta_i} \cdots x_i$.
\end{enumerate}
\end{lemma}

\begin{proof}
By induction on $i$. Base case $i = 0$: stack $[\epsilon]$ (form $\mathbf{U} = \epsilon$). Step: appending $x_i \in \Sigma_{\mathrm{ord}}$ extends $\mathbf{U}$, $\mathbf{q}$, or $\mathbf{a}$; appending $\CALL$ or $\ANSWER$ requires the first form (legality forbids nested unfinished blocks) and produces the second or third; $\UNCALL$ and $\UNANSWER$ contract the second or third back to the first and push/pop. Clauses 3--4 follow from the three-form classification.
\end{proof}

\paragraph{Owner view.}
For $x_j \in \Sigma_{\mathrm{ord}}$, define the \emph{owner frame id}
\begin{equation}
\omega_j := \begin{cases}
\lambda_j, & \lambda_j > 0 \text{ and } x_{\lambda_j} = \CALL, \\
\mathrm{par}_{\lambda_j}, & \lambda_j > 0 \text{ and } x_{\lambda_j} = \ANSWER, \\
F_j, & \text{otherwise},
\end{cases}
\end{equation}
and the \emph{owner imagined input}
\begin{equation}
O_j := \bigl(x_s : 1 \le s \le j,\; x_s \in \Sigma_{\mathrm{ord}},\; \omega_s = \omega_j\bigr),
\end{equation}
so that $x_j$ occupies position $|O_j|$ within $O_j$.

\paragraph{Pending view.}
Define the \emph{pending set}
\begin{equation}
\mathcal{P} := \{j : x_j \in \{\CALL, \ANSWER\}\} \cup \{j : x_j \in \Sigma_{\mathrm{ord}},\; \lambda_j > 0,\; x_{\lambda_j} \in \{\CALL, \ANSWER\}\}.
\end{equation}
For $j \in \mathcal{P}$, let $\mathrm{host}_j := F_{\beta_j}$ be the host-frame id, and the \emph{pending imagined input}
\begin{equation}
P_j := \bigl(x_s : s < \beta_j,\; x_s \in \Sigma_{\mathrm{ord}},\; \omega_s = \mathrm{host}_j\bigr) \concat x_{\beta_j} \cdots x_j.
\end{equation}

\begin{lemma}[Explicit replay formula]
\label{lem:explicit-replay-close}
For every legal prefix $\mathbf{x}_{1:i}$, let $C_i := \{j \le i : x_j \in \Sigma_{\mathrm{ord}},\; \omega_j = F_i\}$. Then
\begin{equation}
\maskrec(\mathbf{x}_{1:i}) = \begin{cases}
O_{j_i^*}, & m_i = \bot,\ C_i \neq \varnothing,\ j_i^* := \max C_i, \\
P_i, & m_i \neq \bot, \\
\epsilon, & m_i = \bot,\ C_i = \varnothing.
\end{cases}
\end{equation}
\end{lemma}

\begin{proof}
If $m_i = \bot$: \Cref{lem:replay-normal-form} gives top frame $\mathbf{U}$, i.e., the ordered ordinary tokens of frame $F_i$, namely $C_i$; if $C_i \neq \varnothing$, this is $O_{j_i^*}$. If $m_i \neq \bot$: the top frame is the committed host body followed by the unfinished suffix, i.e., $P_i$.
\end{proof}

\subsection[Simulating One Layer]{Simulating One Layer of $\mathcal{A}$}
\label{app:lifted-sim}

We now simulate each layer of $\mathcal{A}$ in parallel on both imagined inputs. The owner state updates at every ordinary token, including call- and return-payload tokens inside an unfinished block; the pending state updates only at positions in the unique unfinished suffix.

\begin{lemma}[Bookkeeping labels are computable]
\label{lem:bookkeeping-labels}
All labels used below---$\lambda, c, F, \mathrm{par}, m, \beta, \omega, \mathrm{host}$, the allowed-position sets $J_i^O, J_i^{\mathrm{host}}, J_i^{\mathrm{blk}}$ (defined below), and the host-empty indicator $\mathsf{HE}_i$---are computable by a constant-depth causal Transformer block using $O(1)$ numeric channels, each storing an $O(\log n)$-bit integer or a rational of polynomial bit complexity~\citep{yang2025pencil}.
\end{lemma}

The point of this lemma is only that these quantities are deterministic bookkeeping computed from the control tokens; they are not learned parameters.

\begin{proof}
$\lambda_i, \beta_i$ select the most recent control-token positions satisfying fixed conditions; $c_i$ is a prefix sum; $F_i, \mathrm{par}_i$ select the most recent matching frame ids among $O(\log n)$-bit integers; $\omega_j$ is local given $F_j, \lambda_j, \mathrm{par}_{\lambda_j}$; $\mathrm{host}_j = F_{\beta_j}$ is local. The host-empty indicator is the complement of an OR over earlier positions:
\begin{equation}
\mathsf{HE}_i \;=\; 1 - \mathbf{1}\bigl[\exists s < \beta_i:\; x_s \in \Sigma_{\mathrm{ord}}\ \wedge\ \omega_s = \mathrm{host}_i\bigr],
\end{equation}
and the existential is the OR of $\mathbf{1}[s < \beta_i]\cdot\mathbf{1}[x_s \in \Sigma_{\mathrm{ord}}]\cdot\mathbf{1}[\omega_s = \mathrm{host}_i]$ over earlier positions $s$; the final complement is local. The three allowed-position sets are Boolean combinations of integer equalities and order comparisons among these labels. Each of these operations is known to be implementable in constant depth~\citep{yang2025pencil}.
\end{proof}

\paragraph{Two local subroutines.}
The lifted layer uses only two simple attention subroutines. First, we restrict a head to the positions that are allowed to appear in the simulated frame. Second, when the simulated frame is a host frame followed by an unfinished suffix, we merge the best position from the host part with the best position from the suffix part.

\paragraph{Restricting attention by penalties.}
To make a head ignore disallowed positions without changing how it ranks allowed positions, we add a large negative penalty only to the disallowed positions. Concretely, suppose an allowed-position predicate $\mathsf{Allow}(i,j)$ has constant-width query/key features $a_i,b_j$, computable from the bookkeeping labels, such that
\begin{align}
\langle a_i,b_j\rangle=0
&\qquad\text{whenever }\mathsf{Allow}(i,j)=1,\\
\langle a_i,b_j\rangle\le -1
&\qquad\text{whenever }\mathsf{Allow}(i,j)=0.
\end{align}

\begin{lemma}[Restricted rightmost head]
\label{lem:restricted-rightmost-head}
Fix query, key, and value streams $Q_i,K_j,V_j$, and an allowed-position predicate $\mathsf{Allow}(i,j)$ with penalty features $a_i,b_j$ as above. Suppose every original score satisfies
\begin{equation}
|\langle Q_i,K_j\rangle|\le \Gamma_{\mathcal{A}}(i).
\end{equation}
Let $M_{\mathcal{A}}(i):=2\Gamma_{\mathcal{A}}(i)+1$. Then rightmost-hard attention with score
\begin{equation}
\widetilde S_{ij}
:=
\langle Q_i,K_j\rangle + M_{\mathcal{A}}(i)\langle a_i,b_j\rangle
\end{equation}
selects exactly the rightmost maximizer of $\langle Q_i,K_j\rangle$ among allowed keys, whenever there is at least one allowed key.
\end{lemma}

\begin{proof}
Allowed keys receive zero shift, so their relative ordering and rightmost tie rule are unchanged. Every disallowed key receives a penalty at most $-M_{\mathcal{A}}(i)$. Since $M_{\mathcal{A}}(i)>2\Gamma_{\mathcal{A}}(i)$, every allowed key strictly dominates every disallowed key. Therefore the rightmost-hard winner is exactly the rightmost maximizer among allowed keys.
\end{proof}

\noindent The allowed-position sets $J_i^O, J_i^{\mathrm{host}}, J_i^{\mathrm{blk}}$ used below have such penalty features. They are Boolean combinations of integer equalities and order comparisons among $\lambda,\beta,\omega,\mathrm{host},i,j$, all stored in $O(\log n)$-bit channels. Equalities are enforced by quadratic penalties, inequalities by fixed integer-comparison penalties, and conjunctions by summing penalties.

\begin{lemma}[Merging the host part and the unfinished suffix]
\label{lem:host-block-merge}
Let $Y = Y^H \concat Y^K$ with $Y^H, Y^K$ both non-empty and every position of $Y^K$ strictly right of every position of $Y^H$. Suppose two applications of \Cref{lem:restricted-rightmost-head} return the rightmost-hard (value, key) winners $(v^H, k^H)$ on $Y^H$ and $(v^K, k^K)$ on $Y^K$. For a query $q$, set $s^H := \langle q, k^H\rangle$, $s^K := \langle q, k^K\rangle$, and
\begin{equation}
v := \begin{cases} v^H, & s^H > s^K, \\ v^K, & s^H \le s^K. \end{cases}
\end{equation}
Then $v$ equals the rightmost-hard output on $Y$ with query $q$.
\end{lemma}

\begin{proof}
Rightmost-hard on $Y$ selects the rightmost maximizer. Since $Y^K$ is strictly right of $Y^H$, any cross-tie resolves to the block; the strict inequality $s^H > s^K$ lets host win only when host strictly dominates.
\end{proof}

Fix any architecture $\mathcal{A}$ of depth $L = L(\mathcal{A})$ realizing the local recursive rule $\frcm$.

\paragraph{Lifted state.}
At each position $j$ and layer $\ell$, $\mathcal{A}'$ maintains
\begin{equation}
\widetilde{H}_j^{(\ell)} = \bigl[H_j^{O,(\ell)},\; H_j^{P,(\ell)},\; W_j^{(\ell)},\; M_j\bigr] \in \mathbb{R}^{d^{(\ell)}} \oplus \mathbb{R}^{d^{(\ell)}} \oplus \mathbb{R}^{O(d^{(\ell)})} \oplus \mathbb{R}^{O(1)},
\end{equation}
where $W^{(\ell)}$ stores temporary quantities used inside the layer and cleared afterward, and $M$ stores the bookkeeping labels of \Cref{lem:bookkeeping-labels}.

\paragraph{Initialization.}
$H_j^{O,(0)} := \mathsf{TE}(x_j) + \mathsf{PE}(|O_j|)$ for $x_j \in \Sigma_{\mathrm{ord}}$; $H_j^{P,(0)} := \mathsf{TE}(x_j) + \mathsf{PE}(|P_j|)$ for $j \in \mathcal{P}$; other state entries are zero. Since the positional features can be evaluated at the simulated positions $|O_j|$ and $|P_j|$, these initial states are constant-depth computable from $\mathbf{x}$ and the bookkeeping labels.

\paragraph{Bounding scores before adding penalties.}
Under the bounded-precision parameter model used in \Cref{sec:id-generalization}, each independently stored parameter is stored with $b$ bits. For a fixed architecture $\mathcal{A}$, there are constants $K_{\mathcal{A}},\kappa_{\mathcal{A}}$, depending only on $\mathcal{A}$ and the allowed $b$-bit parameter range, such that every query--key score of $\mathcal{A}$ on length-$i$ inputs has magnitude at most
\begin{equation}
\Gamma_{\mathcal{A}}(i):=K_{\mathcal{A}}(i+1)^{\kappa_{\mathcal{A}}}.
\end{equation}
We use the fixed penalty size $M_{\mathcal{A}}(i):=2\Gamma_{\mathcal{A}}(i)+1$ in \Cref{lem:restricted-rightmost-head}. This penalty is large enough to make every disallowed position lose, but it leaves the scores of allowed positions unchanged.

\paragraph{Allowed positions.}
The lifted head restricts attention to the positions that would be present in the corresponding imagined input. For the owner lane,
\begin{equation}
J_i^O
:=
\{j\le i:\ x_j\in\Sigma_{\mathrm{ord}},\ \omega_j=\omega_i\}.
\end{equation}
For the pending lane, the active frame is a concatenation of the host frame and the unfinished suffix. These two pieces use
\begin{equation}
\begin{aligned}
J_i^{\mathrm{host}}
&:=
\{j<\beta_i:\ x_j\in\Sigma_{\mathrm{ord}},\ \omega_j=\mathrm{host}_i\},
\\
J_i^{\mathrm{blk}}
&:=
\{j:\ \beta_i\le j\le i,\ \beta_j=\beta_i\}.
\end{aligned}
\end{equation}
The set $J_i^O$ is nonempty whenever the owner lane is active, and $J_i^{\mathrm{blk}}$ is nonempty whenever $i\in\mathcal{P}$.

\paragraph{Concrete form of one lifted head.}
Fix layer $\ell$ and original head $a$ of $\mathcal{A}$. We now spell out how this one original head is called on the two lanes. Its learned block consists of
\begin{equation}
W_Q^{(\ell,a)},\quad W_K^{(\ell,a)},\quad W_V^{(\ell,a)},\quad W_O^{(\ell,a)},
\end{equation}
together with the corresponding biases. At layer $\ell-1$, the lifted architecture maintains two lanes,
\begin{equation}
H^{O,(\ell-1)}
\qquad\text{and}\qquad
H^{P,(\ell-1)}.
\end{equation}
The owner application of head $(\ell,a)$ uses
\begin{align}
q_i^O &:= W_Q^{(\ell,a)}H_i^{O,(\ell-1)}+b_Q^{(\ell,a)},\\
k_j^O &:= W_K^{(\ell,a)}H_j^{O,(\ell-1)}+b_K^{(\ell,a)},\\
v_j^O &:= W_V^{(\ell,a)}H_j^{O,(\ell-1)}+b_V^{(\ell,a)}.
\end{align}
The host part of the pending application uses
\begin{align}
q_i^P &:= W_Q^{(\ell,a)}H_i^{P,(\ell-1)}+b_Q^{(\ell,a)},\\
k_j^H &:= W_K^{(\ell,a)}H_j^{O,(\ell-1)}+b_K^{(\ell,a)},\\
v_j^H &:= W_V^{(\ell,a)}H_j^{O,(\ell-1)}+b_V^{(\ell,a)}.
\end{align}
The suffix-block part of the pending application uses the same pending query and
\begin{equation}
\begin{aligned}
k_j^K &:= W_K^{(\ell,a)}H_j^{P,(\ell-1)}+b_K^{(\ell,a)},\\
v_j^K &:= W_V^{(\ell,a)}H_j^{P,(\ell-1)}+b_V^{(\ell,a)}.
\end{aligned}
\end{equation}
In the unfolded architecture, the owner, host, and suffix-block applications are independent copied blocks. We write them using the same symbols $W_Q^{(\ell,a)}, W_K^{(\ell,a)}, W_V^{(\ell,a)}, W_O^{(\ell,a)}$ only to keep notation light; in the simulator parameter setting, all these copies are assigned the corresponding original values from $\mathcal{A}$.

Here $\mathrm{RHard}_{j\in J}$ denotes rightmost-hard attention restricted to positions in $J$.
When the owner lane is active, the owner output is
\begin{equation}
\begin{aligned}
u_i^{O,a}
&:=
\mathrm{RHard}_{j\in J_i^O}
\langle q_i^O,k_j^O\rangle\, v_j^O,\\
z_i^{O,a}&:=W_O^{(\ell,a)}u_i^{O,a}+b_O^{(\ell,a)}.
\end{aligned}
\end{equation}
For the pending lane, we compute the best block candidate:
\begin{equation}
(u_i^K,k_i^K)
:=
\mathrm{RHard}_{j\in J_i^{\mathrm{blk}}}
\langle q_i^P,k_j^K\rangle\, (v_j^K,k_j^K).
\end{equation}
If $J_i^{\mathrm{host}}$ is nonempty, we also compute the best host candidate:
\begin{equation}
(u_i^H,k_i^H)
:=
\mathrm{RHard}_{j\in J_i^{\mathrm{host}}}
\langle q_i^P,k_j^H\rangle\, (v_j^H,k_j^H).
\end{equation}
A fixed pointwise merge gate then computes
\begin{equation}
\begin{aligned}
s_i^H&:=\langle q_i^P,k_i^H\rangle,\\
s_i^K&:=\langle q_i^P,k_i^K\rangle,
\end{aligned}
\end{equation}
and sets
\begin{equation}
u_i^P:=
\begin{cases}
 u_i^H, & s_i^H>s_i^K,\\[2pt]
 u_i^K, & s_i^H\le s_i^K.
\end{cases}
\end{equation}
If $J_i^{\mathrm{host}}$ is empty, the fixed gate skips the host candidate and sets $u_i^P:=u_i^K$.
The pending output of head $(\ell,a)$ is then
\begin{equation}
z_i^{P,a}:=W_O^{(\ell,a)}u_i^P+b_O^{(\ell,a)}.
\end{equation}

Summing over all heads and applying the copied layer-$\ell$ FFN blocks to the two lanes gives
\begin{equation}
\begin{aligned}
\widetilde H_i^{O}&:=H_i^{O,(\ell-1)}+\sum_a z_i^{O,a},\\
H_i^{O,(\ell)}&:=\mathsf{FFN}^{(\ell)}(\widetilde H_i^O),
\end{aligned}
\end{equation}
and
\begin{equation}
\begin{aligned}
\widetilde H_i^{P}&:=H_i^{P,(\ell-1)}+\sum_a z_i^{P,a},\\
H_i^{P,(\ell)}&:=\mathsf{FFN}^{(\ell)}(\widetilde H_i^P).
\end{aligned}
\end{equation}
The owner and pending FFN applications are independent copied blocks in $\mathcal{A}'$, again written as $\mathsf{FFN}^{(\ell)}$ for readability and assigned the same value as the original FFN block in the simulator parameter setting. At positions where a lane is inactive, fixed gates keep that lane at zero. At an ordinary token $x_i \in \Sigma_{\mathrm{ord}}$ with $m_i \neq \bot$ (a call- or return-payload), both lanes update.

\begin{lemma}[Layer-wise invariant]
\label{lem:layerwise-close}
The invariant in \Cref{eq:lifted-invariant} holds at every layer $\ell \in \{0, \ldots, L\}$.
\end{lemma}

\begin{proof}
By induction on $\ell$. Base ($\ell = 0$): from initialization and the positional features at the simulated positions. Step: for $x_j \in \Sigma_{\mathrm{ord}}$, \Cref{lem:restricted-rightmost-head} applied to the owner application matches the original head on $O_j$ at position $|O_j|$. For $j \in \mathcal{P}$: if $\mathsf{HE}_j = 1$, the host half of $P_j$ is empty and the pending computation reduces to the suffix block, so the block candidate is exactly the rightmost-hard winner on $P_j$; if $\mathsf{HE}_j = 0$, both halves are non-empty and \Cref{lem:host-block-merge} applies, with the fixed merge gate implementing the rightmost-hard choice across the host and suffix block. Because the copied attention projections, output projections, and FFN blocks are assigned the same numerical values as the corresponding original blocks of $\mathcal{A}$, the layer update is exactly the original layer of $\mathcal{A}$ run on the relevant imagined input. Hence the invariant is preserved.
\end{proof}

\subsection{Readout and Resource Bounds}
\label{app:readout}

\begin{proof}[Proof of \Cref{thm:cot-simulation}]
Let $\mathcal{A}$ realize the local recursive rule $\frcm$ with depth $L$. Construct $\mathcal{A}'$ in three stages: (1) prepend the constant-depth bookkeeping block of \Cref{lem:bookkeeping-labels}; (2) replace each of the $L$ layers of $\mathcal{A}$ by its lifted version (\Cref{app:lifted-sim}), so that \Cref{lem:layerwise-close} gives the invariant in \Cref{eq:lifted-invariant} at $\ell = L$; (3) apply the readout below.

\paragraph{Readout.} At the final position $n$, select
\begin{equation}
z_n := \begin{cases}
H_{j_n^*}^{O,(L)}, & m_n = \bot,\ C_n \neq \varnothing,\ j_n^* := \max C_n, \\
H_n^{P,(L)}, & m_n \neq \bot, \\
H_1^{O,(L)}, & m_n = \bot,\ C_n = \varnothing \ \text{(sentinel branch)}.
\end{cases}
\end{equation}
By \Cref{lem:layerwise-close} and \Cref{lem:explicit-replay-close}, $z_n = H^{\mathcal{A},L}_{|\maskrec(\mathbf{x})|}(\maskrec(\mathbf{x}))$ in the first two cases, and $z_n = H^{\mathcal{A},L}_1(\BOS)$ in the third, realizing $\frcm(\epsilon) := \frcm(\BOS)$. A rightmost exact-match head copies $z_n$ to position $n$, and a copied decoder block, assigned the same numerical value as the decoder of $\mathcal{A}$, yields $\frcm(\maskrec(\mathbf{x}))$.

\paragraph{Learned-parameter accounting.}
The lifted architecture does more computation than $\mathcal{A}$, because each layer keeps owner and pending states and applies copied versions of the original heads and FFN blocks to the relevant lanes and subroutines. Under the unfolded convention of \Cref{app:lane-wise-architecture}, these copies are counted as independent learned parameters. For the simulator witnessing \Cref{thm:cot-simulation}, the copied parameters are assigned the same numerical values as the corresponding parameters of $\mathcal{A}$.

Each original learned block is copied only a constant number of times by the fixed lifting construction: the owner, host, and suffix-block applications for attention heads, the owner and pending FFN applications, and the final decoder copy. Therefore
\begin{equation}
W(\mathcal{A}')=O(W(\mathcal{A})).
\end{equation}
The hidden constant depends only on the fixed lifting construction, not on the recursive task, the training sample, or the parameter setting of $\mathcal{A}$. The bookkeeping block, allowed-position sets, exact-match selectors, copy operations, and merge gates are fixed non-learned parts of the architecture.

The depth overhead remains constant:
\begin{equation}
L(\mathcal{A}')=L(\mathcal{A})+O(1).
\end{equation}
\end{proof}

%% file: sections/appendix_iid.tex
\section[Proofs for IID Generalization]{Proofs for \Cref{sec:id-generalization} (IID Generalization)}
\label{app:proof-iid}
\label{sec:appendix-iid}

Throughout this appendix, $\mathcal{D}$ gives probability one to input sequences that can arise during recursive generation, as formalized by \Cref{def:legal-prefix}.

\paragraph{Architecture coding convention.}
The Transformer appendix specifies architectures by finite data: depth, layer dimensions, numbers of heads, feed-forward dimensions, fixed positional features, and the fixed non-learned operations used by the constructions. Fix any prefix universal decoder $U$ for such finite architecture specifications, for instance a prefix universal Turing machine whose outputs are parsed as such specifications. For each architecture $\mathcal A$, let $\operatorname{code}(\mathcal A)$ be a shortest $U$-program that outputs $\mathcal A$, with ties broken deterministically. A full finite-precision Transformer description consists of $\operatorname{code}(\mathcal A)$ followed by the $b$-bit encodings of the $W_\mathcal A$ independently stored learned parameters.

This gives the Kraft-valid complexity measure in \Cref{eq:transformer-description-length}: the chosen architecture descriptions form a prefix-free set because they are selected from the prefix domain of $U$. For each fixed $\mathcal A$, there are at most $2^{bW_\mathcal A}$ parameter strings of length $bW_\mathcal A$, so the total Kraft mass contributed by this architecture is at most $2^{-|\operatorname{code}(\mathcal A)|}$. Summing over architectures is bounded by one.

\begin{lemma}[Simulation Coding Overhead]
\label{lem:sim-overhead}
Let $\mathcal{A}'$ be the simulator architecture produced from $\mathcal A$ in \Cref{thm:cot-simulation}. Then
\begin{align}
L_{\mathcal{A}'} = L_\mathcal{A}+O(1),
\qquad
W_{\mathcal{A}'} \le C_W W_\mathcal{A},
\end{align}
and
\begin{align}
|\operatorname{code}(\mathcal A')|
\le
|\operatorname{code}(\mathcal A)|+C_{\mathrm{arch}}
\end{align}
for constants independent of the task and the training sample. Consequently, if a local rule $g$ is realized by $(\mathcal A,\theta)$, the simulator gives $(\mathcal A',\theta')$ realizing $g\circ\maskrec$ on legal recursive prefixes with
\begin{align}
|\operatorname{code}(\mathcal A')|+bW_{\mathcal A'}
\le
C\bigl(|\operatorname{code}(\mathcal A)|+bW_\mathcal A\bigr)+C_0 .
\end{align}
\end{lemma}

\begin{proof}
The agreement on legal recursive prefixes and the depth statement are \Cref{thm:cot-simulation}. For the parameter count, the construction in \Cref{app:proof-cot-simulation} uses only a constant number of copied applications of each original learned block under the unfolded convention; bookkeeping, copy, mask, and merge operations are fixed and non-learned. Hence $W_{\mathcal A'}\le C_W W_\mathcal A$.

For the architecture code, the simulator architecture is obtained from $\mathcal A$ by a fixed computable transformation $S$, so $\mathcal A'=S(\mathcal A)$. Given any $U$-program for $\mathcal A$, a constant-size wrapper can decode $\mathcal A$ and output $S(\mathcal A)$. By universality of $U$, this gives
\begin{align}
|\operatorname{code}(\mathcal A')|
\le
|\operatorname{code}(\mathcal A)|+C_{\mathrm{arch}} .
\end{align}
Combining this with the parameter-count bound gives the final inequality.
\end{proof}

\begin{proof}[Proof of \Cref{cor:cot-transformer-dominance}]
Fix a complete-sequence predictor $f$ with finite $|f|_{\ell^{\maskrec}}$. By definition of the lifted length, choose a shortest local representative $g$ such that $f=g\circ\maskrec$ and $|g|_\ell=|f|_{\ell^{\maskrec}}$. Choose a shortest finite-precision Transformer realization $(\mathcal A,\theta)$ of $g$. By \Cref{lem:sim-overhead}, the simulator produces a CoT predictor $\tilde f$ that agrees with $g\circ\maskrec$, and hence with $f$, on legal recursive prefixes. Moreover,
\begin{align}
|\tilde f|_\ell
&\le
|\operatorname{code}(\mathcal A')|+bW_{\mathcal A'}\\
&\le
C\bigl(|\operatorname{code}(\mathcal A)|+bW_\mathcal A\bigr)+C_0\\
&\le
C|f|_{\ell^{\maskrec}}+C_0 .
\end{align}
\end{proof}

\begin{proof}[Proof of \Cref{thm:iid-sample-complexity}]
Let $S\subseteq\Omega$ be the set of input sequences that can arise during recursive generation. By assumption, a draw $\mathbf{x}\sim\mathcal{D}$ lies in $S$ with probability one. The complete-sequence target is
\begin{align}
\fcot(\mathbf{x}) := \frcm(\maskrec(\mathbf{x})),
\end{align}
and
\begin{align}
\mathcal{L}_{\mathcal{D}}(h)
=
\Pr_{\mathbf{x}\sim\mathcal{D}}[h(\mathbf{x})\neq \fcot(\mathbf{x})].
\end{align}

\textbf{Lifted recursive description.}
Apply \Cref{thm:mdl-gen-bound} under the lifted description language $\ell^{\maskrec}$ with target $\fcot$ and confidence $\delta/2$. Since $\fcot$ itself has lifted length $|\fcot|_{\ell^{\maskrec}}$, the MDL output
\begin{align}
\hat f_{\maskrec}:=\mathsf{MDL}_{\ell^{\maskrec}}(\mathcal X,\fcot)
\end{align}
satisfies, with probability at least $1-\delta/2$,
\begin{align}
\mathcal{L}_{\mathcal{D}}(\hat f_{\maskrec})
\le
\frac{|\fcot|_{\ell^{\maskrec}}\ln2+\ln(2/\delta)}{m}.
\end{align}

\textbf{CoT description.}
By \Cref{cor:cot-transformer-dominance}, there is a CoT predictor $\tilde f$ agreeing with $\fcot$ on legal recursive prefixes and satisfying
\begin{align}
|\tilde f|_\ell\le C\,|\fcot|_{\ell^{\maskrec}}+C_0 .
\end{align}
Since $\mathcal{D}$ is supported on legal recursive prefixes, $\tilde f$ and $\fcot$ agree almost surely. Apply \Cref{thm:mdl-gen-bound} under $\ell$ with target $\tilde f$ and confidence $\delta/2$. On the probability-one event that all training samples are legal recursive prefixes, the learner below has the same feasible set, and hence the same output under the fixed tie-breaking rule, as $\mathsf{MDL}_{\ell}(\mathcal X,\tilde f)$:
\begin{align}
\hat\fcot:=\mathsf{MDL}_{\ell}(\mathcal X,\fcot),
\end{align}
and the loss against $\tilde f$ equals the loss against $\fcot$. Therefore, with probability at least $1-\delta/2$,
\begin{align}
\mathcal{L}_{\mathcal{D}}(\hat\fcot)
\le
\frac{\bigl(C\,|\fcot|_{\ell^{\maskrec}}+C_0\bigr)\ln2+\ln(2/\delta)}{m}.
\end{align}
A union bound gives both inequalities simultaneously with probability at least $1-\delta$.
\end{proof}

%% file: sections/appendix_proofs.tex
\section{Proofs for Section 5}
\label{sec:proofs-invariance}
\label{sec:appendix-proofs}

This appendix contains the remaining proof used in \Cref{sec:ood-section}. For the general invariance statements, $\mask:\Omega\to\Omega$ denotes an arbitrary idempotent observation function, $\frcm:\Omega\to\Sigma$ is the local target rule, and the induced target on complete sequences is
\begin{align}
\fcot(\mathbf{x}) := \frcm(\mask(\mathbf{x})).
\end{align}
Specializing to the recursive model means setting $\mask=\maskrec$, the active-frame observation function from \Cref{sec:formal-model}. The set $\mathcal{X}\subseteq\Omega$ always denotes the training set of complete sequences.

\subsection{Strict Gap Is Not Necessary}

\Cref{thm:shortcut-main} gives a sufficient condition for non-invariance:
\red{the CoT predictor is strictly shorter under $\ell$ than the lifted MDL output under $\ell^{\mask}$.} The
condition is not necessary. A learner that reads the complete sequence can choose a
\red{non-invariant predictor even when the lifted MDL problem has an equally short
or shorter solution.}

\begin{proposition}[Non-Invariance Without a Strict Shortcut Gap]
\label{prop:notb-and-a}
There exist a local recursive target $\frcm$ and its induced full-sequence target $\fcot=\frcm\circ\maskrec$, a
\red{complexity measure $\ell$}, and a training set of complete sequences $\mathcal{X}\subseteq\Omega$ such that the CoT MDL output
\begin{equation}
\red{\hat\fcot:=\mathsf{MDL}_{\ell}(\mathcal{X},\fcot)}
\end{equation}
is not $\maskrec$-invariant on $\mathcal{X}$, but the strict shortcut gap fails:
\begin{align}
|\hat\fcot|_\ell
\ge
\red{\left|\mathsf{MDL}_{\ell^{\maskrec}}(\mathcal{X},\fcot)\right|_{\ell^{\maskrec}}.}
\end{align}
\end{proposition}

\begin{proof}
Choose a control-free string $\mathbf{v}\in\Sigma_{\mathrm{ord}}^*$ and let
\begin{align}
\mathbf{x}_0 := \CALL\,\mathbf{v}\,\UNCALL.
\end{align}
Under the stack procedure, $\mathbf{x}_0$ opens a child frame initialized with
$\mathbf{v}$, so
\begin{align}
\maskrec(\mathbf{x}_0)=\mathbf{v},
\qquad
\maskrec(\mathbf{v})=\mathbf{v}.
\end{align}
Let $\mathcal{X}:=\{\mathbf{x}_0\}$. Pick two distinct output tokens
$a_0,a_1\in\Sigma$, and define the local target $\frcm$ by
$\frcm(\mathbf{v})=a_0$; its values elsewhere are arbitrary. Then
$\fcot(\mathbf{x}_0)=a_0$ and $\fcot(\mathbf{v})=a_0$.

Now define three hypotheses:
\begin{align}
h_{\mathrm{non}}(\mathbf{u})
&:=
\begin{cases}
a_0, & \mathbf{u}=\mathbf{x}_0,\\
a_1, & \text{otherwise},
\end{cases}
\\
h_{\mathrm{loc}}(\mathbf{u})
&:=
\begin{cases}
a_0, & \mathbf{u}=\mathbf{v},\\
a_1, & \text{otherwise},
\end{cases}
\\
h_{\mathrm{inv}}(\mathbf{u})
&:=
\begin{cases}
a_0, & \maskrec(\mathbf{u})=\mathbf{v},\\
a_1, & \text{otherwise}.
\end{cases}
\end{align}
\red{Assign finite description lengths}
\begin{align}
|h_{\mathrm{loc}}|_\ell=1,\qquad
|h_{\mathrm{non}}|_\ell=2,\qquad
|h_{\mathrm{inv}}|_\ell=3.
\end{align}
\red{All other predictors have length $+\infty$.} These lengths satisfy the Kraft inequality, since
$2^{-1}+2^{-2}+2^{-3}<1$.

For the CoT learner, feasibility means agreeing with $\fcot$ directly on
$\mathcal{X}$. Both $h_{\mathrm{non}}$ and $h_{\mathrm{inv}}$ are feasible, but
$h_{\mathrm{loc}}$ is not because $h_{\mathrm{loc}}(\mathbf{x}_0)=a_1$. Hence
\red{the CoT MDL output under the original description language is}
\begin{align}
\hat\fcot=h_{\mathrm{non}},
\qquad
|\hat\fcot|_\ell=2.
\end{align}
This predictor is not $\maskrec$-invariant on $\mathcal{X}$ because
\begin{align}
\maskrec(\mathbf{x}_0)=\maskrec(\mathbf{v})=\mathbf{v},
\qquad
h_{\mathrm{non}}(\mathbf{x}_0)=a_0\neq a_1=h_{\mathrm{non}}(\mathbf{v}).
\end{align}
It also fails on the OOD sequence $\mathbf{v}$, since
$\fcot(\mathbf{v})=a_0$ while $h_{\mathrm{non}}(\mathbf{v})=a_1$.

\red{For the lifted learner, we use $\ell^{\maskrec}$ on complete-sequence predictors. The invariant predictor satisfies}
\begin{align}
\red{h_{\mathrm{inv}}=h_{\mathrm{loc}}\circ\maskrec,
\qquad
h_{\mathrm{inv}}(\mathbf{x}_0)=a_0=\fcot(\mathbf{x}_0),}
\end{align}
\red{so $h_{\mathrm{inv}}$ is feasible for $\mathsf{MDL}_{\ell^{\maskrec}}(\mathcal{X},\fcot)$. Moreover}
\begin{align}
\red{|h_{\mathrm{inv}}|_{\ell^{\maskrec}}\le |h_{\mathrm{loc}}|_\ell=1.}
\end{align}
\red{No feasible complete-sequence predictor can have lifted length below $1$, because the only finite-length local representative with length $1$ is $h_{\mathrm{loc}}$, and it induces $h_{\mathrm{inv}}$. Thus}
\begin{align}
\red{\mathsf{MDL}_{\ell^{\maskrec}}(\mathcal{X},\fcot)=h_{\mathrm{inv}},
\qquad
\left|\mathsf{MDL}_{\ell^{\maskrec}}(\mathcal{X},\fcot)\right|_{\ell^{\maskrec}}=1.}
\end{align}
Thus the strict shortcut gap does not hold: $2<1$ is false. Nevertheless, the
CoT MDL output is non-invariant and makes an OOD error. This proves that the
strict gap is sufficient, but not necessary, for the failure mode described in
\Cref{sec:ood-section}.
\end{proof}

%% file: sections/appendix_experiments.tex
\section{Experimental Details}
\label{app:exp-details}
\label{sec:appendix-experiments}

\input{sections/appendix_experiments/function_library}
\input{sections/appendix_experiments/setup}
\input{sections/appendix_experiments/iid}
\input{sections/appendix_experiments/ood}
\input{sections/appendix_experiments/shortcut}

%% file: sections/appendix_experiments/function_library.tex
\subsection{Function Library}
\label{app:function-library}

The expression generator draws from the typed function library in
\Cref{tab:function-library}. Primitive functions are evaluated directly.
Composite functions are named bodies that expand into calls to other library
functions. Tail-recursive functions are also named bodies, but their bodies may
call the same function again with updated arguments. This last group is what
lets a shallow written expression produce a much longer call/return trace at
execution time. Integer-valued outputs are reduced modulo~$10$; Boolean-valued
outputs are emitted as Boolean tokens.

\begin{table}[h]
\centering
\small
\setlength{\tabcolsep}{4pt}
\renewcommand{\arraystretch}{1.08}
\caption{Function library used by the expression generator.}
\label{tab:function-library}
\begin{tabular}{@{}p{0.21\linewidth}p{0.20\linewidth}p{0.50\linewidth}@{}}
\toprule
Function & Type & Body / definition \\
\midrule
\multicolumn{3}{@{}l}{\textbf{Primitive functions}} \\
\texttt{add} & int, int $\to$ int & $(x+y) \bmod 10$ \\
\texttt{sub} & int, int $\to$ int & $(x-y) \bmod 10$ \\
\texttt{multiply} & int, int $\to$ int & $(x\cdot y) \bmod 10$ \\
\texttt{diff} & int, int $\to$ int & $|x-y| \bmod 10$ \\
\texttt{square} & int $\to$ int & $x^2 \bmod 10$ \\
\texttt{double} & int $\to$ int & $2x \bmod 10$ \\
\texttt{min} & int, int $\to$ int & $\min(x,y)$ \\
\texttt{max} & int, int $\to$ int & $\max(x,y)$ \\
\texttt{less} & int, int $\to$ bool & $\mathbf{1}\{x<y\}$ \\
\texttt{is\_even} & int $\to$ bool & $\mathbf{1}\{x \bmod 2 = 0\}$ \\
\texttt{if\_then\_else} & bool, $\tau$, $\tau$ $\to$ $\tau$ &
returns the second argument if the condition is true, otherwise the third \\
\midrule
\multicolumn{3}{@{}l}{\textbf{Composite functions}} \\
\texttt{triple\_add}$(a,b,c)$ & int, int, int $\to$ int &
\(\texttt{add}(\texttt{add}(a, b),\, c)\) \\
\texttt{sum\_of\_squares}$(a,b)$ & int, int $\to$ int &
\(\texttt{add}(\texttt{square}(a),\, \texttt{square}(b))\) \\
\makecell[l]{\texttt{diff\_of\_squares}\\$(a,b)$} & int, int $\to$ int &
\(\texttt{sub}(\texttt{square}(a),\, \texttt{square}(b))\) \\
\texttt{clamp}$(v,l,u)$ & int, int, int $\to$ int &
\(\texttt{min}(\texttt{max}(v, l),\, u)\) \\
\makecell[l]{\texttt{manhattan}\\$(x_1,y_1,x_2,y_2)$} & int, int, int, int $\to$ int &
\(\texttt{add}(\texttt{diff}(x_1, x_2),\, \texttt{diff}(y_1, y_2))\) \\
\texttt{is\_in\_range}$(v, l, u)$ & int, int, int $\to$ bool &
\(\texttt{less}(\texttt{sub}(v, l),\, \texttt{sub}(u, l))\) \\
\makecell[l]{\texttt{point\_in\_rect}\\$(x,y,x_1,y_1,x_2,y_2)$} & six int $\to$ bool &
\(\texttt{if\_then\_else}\bigl(\texttt{is\_in\_range}(x, x_1, x_2),\)\\
& &
\(\hphantom{\texttt{if\_then\_else}\bigl(}\texttt{is\_in\_range}(y, y_1, y_2),\, \texttt{False}\bigr)\) \\
\midrule
\multicolumn{3}{@{}l}{\textbf{Tail-recursive functions}} \\
\texttt{accum\_sum}$(n,\mathrm{acc})$ & int, int $\to$ int &
\(\texttt{if\_then\_else}\bigl(\texttt{less}(n, 1),\, \mathrm{acc},\)\\
& &
\(\hphantom{\texttt{if\_then\_else}\bigl(}\texttt{accum\_sum}(\texttt{sub}(n,1),\, \texttt{add}(\mathrm{acc}, n))\bigr)\) \\
\texttt{factorial}$(n,\mathrm{acc})$ & int, int $\to$ int &
\(\texttt{if\_then\_else}\bigl(\texttt{less}(n, 2),\, \mathrm{acc},\)\\
& &
\(\hphantom{\texttt{if\_then\_else}\bigl(}\texttt{factorial}(\texttt{sub}(n,1),\)\\
& &
\(\hphantom{\texttt{if\_then\_else}\bigl(\texttt{factorial}(}\texttt{multiply}(\mathrm{acc}, n))\bigr)\) \\
\texttt{fibonacci}$(n,a,b)$ & int, int, int $\to$ int &
\(\texttt{if\_then\_else}\bigl(\texttt{less}(n, 1),\, a,\)\\
& &
\(\hphantom{\texttt{if\_then\_else}\bigl(}\texttt{fibonacci}(\texttt{sub}(n,1),\, b,\)\\
& &
\(\hphantom{\texttt{if\_then\_else}\bigl(\texttt{fibonacci}(}\texttt{add}(a, b))\bigr)\) \\
\bottomrule
\end{tabular}
\end{table}

\paragraph{\red{Worked trace}.}
\label{app:trace-example}
We illustrate how a written expression turns into the two training views using
the example
\(\texttt{add}(\texttt{sum\_of\_squares}(2, 3),\, \texttt{double}(4))\), whose
ground-truth value is $1$.
\Cref{tab:trace-example} shows the complete flattened trace on the left
and the active-frame decomposition on the right. This execution has $100$
flattened tokens under the tokenizer from \Cref{app:tokenization}, $11$
active-frame examples, and maximum stack depth $4$. The
example is small enough to print, but it already shows the main mechanism: the
composite call
\(\texttt{sum\_of\_squares}(2,3)\) expands into
\(\texttt{add}(\texttt{square}(2), \texttt{square}(3))\), and those nested calls
create separate frames in the execution trace. The two views encode the same
execution but expose different information at training time:
\begin{itemize}[leftmargin=18pt, itemsep=1pt]
\item The CoT learner trains on the entire flattened trace as one long
      sequence, so at every prediction step the model sees \emph{all} earlier
      tokens of the trace.
\item The RM learner trains on the eleven active-frame examples
      \red{as separate examples}. Each example contains the visible context (the parent
      tokens before the boundary $m$) followed by what is generated inside
      that frame. Loss is applied only to the latter. The RM learner thus
      never sees, at the same training step, tokens belonging to a sibling
      frame or a popped child frame.
\end{itemize}
In the right-hand column, $d$ is stack depth and $m$ is the visible-context
boundary for that frame.

\begin{table}[h]
\centering
\scriptsize
\setlength{\tabcolsep}{3pt}
\renewcommand{\arraystretch}{1.05}
\caption{\red{Worked trace: flattened CoT view and active-frame RM view.}}
\label{tab:trace-example}
\begin{tabular}{@{}p{0.46\linewidth}@{\hspace{6pt}}p{0.50\linewidth}@{}}
\toprule
\textbf{(A) Flattened trace (CoT view)} & \textbf{(B) Active-frame decomposition (RM view)} \\
\midrule
\begin{minipage}[t]{\linewidth}\ttfamily\raggedright
add ( sum\_of\_squares ( 2 , 3 ) , double ( 4 ) ) x := \textcolor{BrickRed!75!black}{<call>} sum\_of\_squares ( 2 , 3 ) \textcolor{BrickRed!75!black}{</call>} a := 2 b := 3 return \textcolor{BrickRed!75!black}{<return>} \textcolor{BrickRed!75!black}{<call>} add ( square ( 2 ) , square ( 3 ) ) \textcolor{BrickRed!75!black}{</call>} x := \textcolor{BrickRed!75!black}{<call>} square ( 2 ) \textcolor{BrickRed!75!black}{</call>} x := 2 return \textcolor{BrickRed!75!black}{<return>} 4 \textcolor{BrickRed!75!black}{</return>} y := \textcolor{BrickRed!75!black}{<call>} square ( 3 ) \textcolor{BrickRed!75!black}{</call>} x := 3 return \textcolor{BrickRed!75!black}{<return>} 9 \textcolor{BrickRed!75!black}{</return>} return \textcolor{BrickRed!75!black}{<return>} 3 \textcolor{BrickRed!75!black}{</return>} \textcolor{BrickRed!75!black}{</return>} y := \textcolor{BrickRed!75!black}{<call>} double ( 4 ) \textcolor{BrickRed!75!black}{</call>} x := 4 return \textcolor{BrickRed!75!black}{<return>} 8 \textcolor{BrickRed!75!black}{</return>} return \textcolor{BrickRed!75!black}{<return>} 1 \textcolor{BrickRed!75!black}{</return>}
\end{minipage}
&
\begin{minipage}[t]{\linewidth}\ttfamily\raggedright\fontsize{6.4}{7.5}\selectfont
\textbf{F1} (d=1, m=8): \\
\textcolor{NavyBlue!70!black}{add ( sum\_of\_squares ( 2 , 3 )} \textbf{|} , double ( 4 ) ) x := \textcolor{BrickRed!75!black}{<call>} sum\_of\_squares ( 2 , 3 ) \textcolor{BrickRed!75!black}{</call>} \\[2pt]
\textbf{F2} (d=2, m=6): \\
\textcolor{NavyBlue!70!black}{sum\_of\_squares ( 2 , 3 )} \textbf{|} a := 2 b := 3 return \textcolor{BrickRed!75!black}{<return>} \textcolor{BrickRed!75!black}{<call>} add ( square ( 2 ) , square ( 3 ) ) \textcolor{BrickRed!75!black}{</call>} \\[2pt]
\textbf{F3} (d=3, m=12): \\
\textcolor{NavyBlue!70!black}{add ( square ( 2 ) , square ( 3 ) )} \textbf{|} x := \textcolor{BrickRed!75!black}{<call>} square ( 2 ) \textcolor{BrickRed!75!black}{</call>} \\[2pt]
\textbf{F4} (d=4, m=4): \\
\textcolor{NavyBlue!70!black}{square ( 2 )} \textbf{|} x := 2 return \textcolor{BrickRed!75!black}{<return>} 4 \textcolor{BrickRed!75!black}{</return>} \\[2pt]
\textbf{F5} (d=3, m=15): \\
\textcolor{NavyBlue!70!black}{add ( square ( 2 ) , square ( 3 ) ) x := 4} \textbf{|} y := \textcolor{BrickRed!75!black}{<call>} square ( 3 ) \textcolor{BrickRed!75!black}{</call>} \\[2pt]
\textbf{F6} (d=4, m=4): \\
\textcolor{NavyBlue!70!black}{square ( 3 )} \textbf{|} x := 3 return \textcolor{BrickRed!75!black}{<return>} 9 \textcolor{BrickRed!75!black}{</return>} \\[2pt]
\textbf{F7} (d=3, m=18): \\
\textcolor{NavyBlue!70!black}{add ( square ( 2 ) , square ( 3 ) ) x := 4 y := 9} \textbf{|} return \textcolor{BrickRed!75!black}{<return>} 3 \textcolor{BrickRed!75!black}{</return>} \\[2pt]
\textbf{F8} (d=2, m=15): \\
\textcolor{NavyBlue!70!black}{sum\_of\_squares ( 2 , 3 ) a := 2 b := 3 return \textcolor{BrickRed!75!black}{<return>} 3} \textbf{|} \textcolor{BrickRed!75!black}{</return>} \\[2pt]
\textbf{F9} (d=1, m=17): \\
\textcolor{NavyBlue!70!black}{add ( sum\_of\_squares ( 2 , 3 ) , double ( 4 ) ) x := 3} \textbf{|} y := \textcolor{BrickRed!75!black}{<call>} double ( 4 ) \textcolor{BrickRed!75!black}{</call>} \\[2pt]
\textbf{F10} (d=2, m=4): \\
\textcolor{NavyBlue!70!black}{double ( 4 )} \textbf{|} x := 4 return \textcolor{BrickRed!75!black}{<return>} 8 \textcolor{BrickRed!75!black}{</return>} \\[2pt]
\textbf{F11} (d=1, m=20): \\
\textcolor{NavyBlue!70!black}{add ( sum\_of\_squares ( 2 , 3 ) , double ( 4 ) ) x := 3 y := 8} \textbf{|} return \textcolor{BrickRed!75!black}{<return>} 1 \textcolor{BrickRed!75!black}{</return>}
\end{minipage}
\\
\bottomrule
\end{tabular}
\end{table}

The vertical bar $\,\vert\,$ in the right-hand column marks the
mask-index boundary $m$: the visible context (in blue) is shown to the
recursive learner but receives no loss signal; the generated tokens to its
right are the loss-bearing targets for that frame.

%% file: sections/appendix_experiments/setup.tex
\subsection{Experimental Setup}
\label{app:experiment-setup}
\label{app:training-details}

This section records the setup shared by all experiments. The three following
appendix subsections describe the parts that differ across the IID, length/depth,
and shortcut studies.

\paragraph{Task and generator.}
\label{app:data-generation}
The task is symbolic expression evaluation. Expressions are built from a typed
function library, listed in \Cref{app:function-library}. Integer-valued
operations are evaluated modulo~$10$, while Boolean-valued operations return
Boolean tokens. The library contains one-step primitives, shallow composite
functions, and tail-recursive functions that can repeatedly call themselves
during execution.

To sample an expression, we choose the root uniformly from the available
function pool. Each argument is then filled independently: with probability
$p_{\mathrm{lit}}=0.4$ we draw a digit literal uniformly from
$\{0,\ldots,9\}$, and otherwise we recursively draw a type-compatible
sub-expression. We enforce type consistency throughout, so Boolean arguments are
filled only by Boolean-returning expressions, and we stop the static recursion at
nesting depth $3$. This bound applies to the written expression, not to the
executed trace: tail-recursive functions can unroll at run time and produce much
longer traces and deeper stacks. The experiment-specific sections below describe
how these generated expressions are split into training, IID validation, and OOD
test sets.

\paragraph{Tokenization and traces.}
\label{app:tokenization}
We use a shared whitespace tokenizer over function names, digit literals,
syntax tokens, Boolean literals, end-of-sequence markers, and the four trace
control tokens $\{\CALL,\UNCALL,\ANSWER,\UNANSWER\}$. Executing an expression
produces a canonical call/return trace. A $\CALL$ block opens a subproblem; an
$\ANSWER$ block returns its value to the parent frame. We measure trace length
as the number of generated tokens in this canonical trace, including syntax,
literal, and control tokens. The maximum recursion depth is the largest stack
depth reached during execution, with the root frame counted as depth~$1$.

\paragraph{Training formats.}
\label{app:training-formats}
CoT and RM train on the same underlying expressions but receive different views
of the corresponding trace. CoT trains on the complete flattened trace with
standard left-to-right next-token prediction. RM trains on active-frame
examples: at each prediction step, the visible context is the current frame, and
the loss is applied only to tokens generated inside that frame. Thus RM uses the
same target computation as CoT, but hides tokens belonging to other active or
completed frames. One expression yields one CoT trace example but many RM
active-frame examples. \red{The experiment-specific sections state whether the
comparison is matched by underlying expressions, by training range, or by
another criterion.}

\paragraph{Models and optimization.}
\label{app:architecture}
\label{app:optimization}
\label{app:convergence}
Unless otherwise stated, both learners use the same decoder-only Transformer
block: $6$ layers, $6$ attention heads, embedding dimension $384$, MLP
expansion factor $4$, learned absolute positional embeddings, GELU activations,
and pre-LayerNorm. With the RM context window this is about $10.65$M parameters;
CoT uses the same block with a $4{,}096$-token positional table, which slightly
increases its positional-embedding parameters and avoids truncating long
flattened traces. The RM format uses a $2{,}048$-token context window. In all
comparisons, the two learners are trained on the same expression pools; the
matching criterion is specified in the relevant experiment subsection.

Optimization uses AdamW with learning rate $10^{-3}$, $\beta_1=0.9$,
$\beta_2=0.99$, weight decay $0.1$, batch size $64$, gradient clipping at
$1.0$, a $2{,}000$-step linear warmup, and cosine decay unless a subsection
states otherwise. For the length/depth and shortcut experiments, we train until
held-out in-distribution accuracy has plateaued for both learners; \red{the IID
sample-complexity experiment instead varies the number of training samples and
uses the best checkpoint found for each sample size.}

\paragraph{Evaluation.}
\label{app:generation}
\label{app:eval-samples}
At test time, each model generates a trace in its native format. CoT generates a
flat trace autoregressively. RM uses a stack driver: when it emits a call, a new
active frame is pushed; when it emits an answer, the returned value is written
back to the parent frame. We extract the final returned value and score it
against the ground-truth expression value: a digit for integer-valued tasks and
a Boolean token for Boolean-valued tasks. RM and CoT are evaluated on the same
test expressions within each split, and confidence intervals, when reported, use
the Wilson score method. \red{The IID sample-complexity experiment also reports
a teacher-forced per-token test accuracy diagnostic, described separately in
\Cref{app:iid-experiment-setup}.}

%% file: sections/appendix_experiments/iid.tex
\subsection{Detailed Setup for IID Compositional Generalization}
\label{app:iid-experiment-setup}
\label{app:iid-sample-complexity}

\paragraph{Split.}
\Cref{sec:exp-iid-compositional} uses an IID random split of the expression
generator from \Cref{app:experiment-setup}. From a seed-$42$ expression pool, we
reserve $5{,}000$ held-out expressions for testing and $1{,}000$ for validation;
the remaining expressions form the training pool. Train and test expressions
are sampled from the same generator, so this experiment tests generalization to
new compositions rather than extrapolation in trace length or recursion depth.

\paragraph{Sample-size sweep.}
We vary the number of training expressions
\[
\begin{aligned}
N\in\{&100,200,300,500,1000,2000,3000,5000,\\
&10000,20000,30000,50000,100000,200000,300000,450000\}.
\end{aligned}
\]
The subsets are nested: the training set for a smaller $N$ is contained in the
training set for a larger $N$. For each $N$, RM and CoT train on the same
underlying expressions in their native formats. RM uses active-frame subsequence
examples; CoT uses the flattened execution trace. Both use the common
$6$-layer Transformer backbone. The CoT format uses a larger context window and
smaller batch size ($4{,}096$ tokens, batch size $32$) than RM ($2{,}048$
tokens, batch size $64$). This IID sweep uses no weight decay.

\paragraph{Training and evaluation.}
For each $(N,\text{format})$, training is allowed to make multiple passes over
the fixed $N$ examples. Training is timeout-based: at regular intervals we
run greedy validation generation on IID held-out expressions and save the
checkpoint with the best validation final-answer accuracy. The selected
checkpoint is then evaluated on the held-out test split. We use three seeds for
the smaller sample sizes and one seed for the largest sample sizes.

The left panel of \Cref{fig:iid-sample-complexity} reports final-answer test
accuracy from greedy autoregressive generation. The end-to-end sweep has test
results through $N=300{,}000$; the per-token sweep also includes $N=450{,}000$.
The right panel reports teacher-forced per-token test accuracy: at each
eligible target position, the model sees the ground-truth prefix and we compare
the argmax next-token prediction with the true next token. For CoT, eligible
positions are the non-padded target positions of the flattened trace after the
prompt. For RM, eligible positions are the loss-bearing active-frame positions
at or after the frame's mask boundary. The per-token confidence bands use a
sample-level bootstrap over test expressions rather than treating all tokens
as independent. The validation split is used only for checkpoint selection;
both reported metrics are computed post hoc on the test split.

%% file: sections/appendix_experiments/ood.tex
\subsection{Detailed Setup for Length and Depth Generalization}
\label{app:length-depth-setup}
\label{app:length-depth-full}

\paragraph{Threshold splits.}
\Cref{sec:exp-length-depth} uses the same generator as the IID experiment, but
splits examples by realized execution properties. For a length threshold $L$,
training and validation expressions have generated trace length $\ell \le L$,
while OOD test expressions have $\ell > L$. Here $\ell$ is \red{executed trace}
length: the number of tokens in the executed call/return trace, not the prompt
length. For a depth threshold $K$, the same construction is applied to maximum
recursion depth $d_{\max}$. Both $\ell$ and $d_{\max}$ are measured after
executing the expression, not from the static written expression.

We draw the sweep data from a fixed $500$k-expression pool generated with
seed~$42$. We sweep eight length thresholds
$L \in \{33, 96, 197, 359, 595, 1060, 1345, 1839\}$
and five depth thresholds $K \in \{3,5,7,9,10\}$. The main-text table uses the
representative splits $L=595$ and $K=10$, and held-out examples from the
training side of each threshold provide the IID check.

For \Cref{tab:per-bin-generalization}, we further group evaluation expressions
by their multiplier relative to the split threshold: $\ell/L$ for the length
split and $d_{\max}/K$ for the depth split. The reported bins are fixed
multiplier intervals, not quantiles. The representative length and depth tables
each use $3{,}100$ evaluation expressions shared by RM and CoT; the smallest
reported bins contain $72$ and $159$ examples, respectively.

\paragraph{Training and evaluation.}
For each split, RM and CoT use the common architecture and training protocol
from \Cref{app:experiment-setup}. We train both learners until held-out
\red{below-threshold validation accuracy} plateaus at at least $99\%$, then evaluate final-answer
accuracy on OOD expressions from the corresponding test side. All accuracy
numbers in \Cref{tab:per-bin-generalization} use the same expression samples for
RM and CoT within each bin.

\paragraph{Frame-width diagnostic.}
The length split has one additional diagnostic. For an expression, let
$w_{\max}$ be the largest token length of any active frame that RM must process
during execution. Let
\begin{equation}
T_{\max}=\max_{x\in\mathcal{X}_{\mathrm{train}}} w_{\max}(x)
\end{equation}
be the largest active-frame width observed on the training side of the split.
Filtering OOD examples by $w_{\max}\le T_{\max}$ asks whether RM's remaining
errors come from genuinely longer traces, or from active frames that are
themselves wider than anything seen in training.

For the representative length split, $T_{\max}=115$. The filter matters only in
the farthest length bin: for traces longer than $5L$, RM accuracy rises from
$76.4\%$ to $85.0\%$ on the retained subset. For the representative depth split,
$T_{\max}=160$ and the filter retains all samples, so the depth shift is not a
single-frame-width artifact. In other words, the final RM drop in the length
table partly reflects active frames that are wider than any training frame,
whereas the depth table isolates a shift in the number of nested frames. Since
the unfiltered numbers are the primary result, we report
this diagnostic in prose rather than as an additional table.

%% file: sections/appendix_experiments/shortcut.tex
\subsection{Detailed Setup for Shortcut Behavior}
\label{app:shortcut-behavior-setup}
\label{app:additional-cases}

\paragraph{Train/OOD construction.}
The shortcut experiments use paired train/OOD distributions. Each construction
defines a shortcut signal $S(\mathbf{x})$ that is visible somewhere in the
complete CoT trace but hidden from the active frame visible to RM. In training,
the data are constructed so that $S(\mathbf{x})$ agrees with the answer. In OOD
evaluation, this coupling is broken while the intended root computation remains
valid. \red{The constructed cases may introduce fixed wrapper or helper
templates at the root; the recursive sub-expressions inside those templates are
sampled from the library in \Cref{app:function-library}.} Except where stated
otherwise, the OOD split keeps the same distribution over recursive
sub-expressions and changes only the argument or filter that made the shortcut
agree with the answer. Thus IID accuracy checks whether the
distribution is learnable, and OOD accuracy plus shortcut rate checks which rule
the model learned.

Let $\widehat{\mathrm{ans}}_{\mathrm{CoT}}(\mathbf{x})$ be CoT's final returned
value. On an OOD sample $\mathcal{X}_{\mathrm{OOD}}$, the shortcut rate is
\begin{align}
\label{eq:shortcut-following}
    \mathrm{follow}_S
    \;:=\;
    \frac{1}{|\mathcal{X}_{\mathrm{OOD}}|}
    \sum_{\mathbf{x}\in\mathcal{X}_{\mathrm{OOD}}}
    \mathbf{1}\!\left\{\widehat{\mathrm{ans}}_{\mathrm{CoT}}(\mathbf{x}) = S(\mathbf{x})\right\}.
\end{align}
All six constructions in \Cref{tab:robustness} are evaluated on $n=300$ IID
and $n=300$ OOD expressions.

\paragraph{Case 1: copy answer from a deep return.}
Let $Z$ be a deep recursive sub-expression with return value $z$, drawn from the
tail-recursive part of the library. The training root template is
\begin{align}
    \texttt{add}\!\left(\texttt{square}(\texttt{square}(Z)),\,Y_{\mathrm{comp}}\right),
    \qquad
    Y_{\mathrm{comp}} := (z-z^4)\bmod 10 .
\end{align}
The training answer is therefore $z$, so $S(\mathbf{x})=z$. OOD resamples the
second argument, so the answer is computed from the root-visible value
$z^4\bmod 10$ and the new argument. CoT can read the deep return token $z$;
RM sees only the transformed root payload. The double-square wrapper is
deliberately non-injective modulo~$10$, so RM cannot recover $z$ from the
root-visible payload by a one-to-one lookup.

\paragraph{Case 2: trace-position parity.}
The root template is $\texttt{is\_even}(\texttt{add}(Z,Y))$. The shortcut
signal $S(\mathbf{x})$ is the parity of the final answer token's position in
the complete trace, encoded as the binary label. Training keeps examples where
this parity agrees with the label; OOD keeps examples where the relationship is
flipped. The marginal distributions of $Z$ and $Y$ are kept the same across the
two splits; only the coupling between label and global answer position changes.
CoT can read the global position signal from the flat trace, while RM sees only
local frame positions because every active frame starts from its own local
offset.

\paragraph{Hidden intermediate literal.}
This construction hides a runtime literal inside an intermediate helper frame.
A helper expression first evaluates a sub-expression $W$, computes
$k=\texttt{eval}(W)\bmod 10$, and tail-calls another subproblem with $k$ as one
of its arguments. The root wraps this helper inside a selector that discards the
helper's returned value before the payload reaches the root. Thus the literal
$k$ appears in the flat trace, but not in the static root expression and not in
the root-frame payload visible to RM. Training uses a compensating argument so
that the answer equals $k$; OOD keeps the same hidden computation but resamples
the compensating argument, so the correct root answer no longer equals $k$.

\paragraph{Long-range aggregate.}
The root template is $\texttt{add}(Z,Y_{\mathrm{comp}})$. Let
$r_1,r_2,r_3$ be the first three integer returns appearing inside $Z$'s trace,
and define
\begin{equation}
S(\mathbf{x})=(r_1+r_2+r_3)\bmod 10 .
\end{equation}
If $z$ is the final return value of $Z$, training chooses
$Y_{\mathrm{comp}}=(S(\mathbf{x})-z)\bmod 10$ so that the root answer equals
$S(\mathbf{x})$. OOD resamples the second argument, breaking this equality. CoT
can read the three deep return tokens from the flat trace; RM sees only the
final return value $z$ of $Z$ at the root. \red{We build the candidate pool for
$Z$ by rejection sampling so that the same final value $z$ occurs with different
shortcut values $S$.} This prevents RM from recovering the shortcut from the
root payload alone.

\paragraph{Near-answer aggregate.}
This construction is the same as the long-range aggregate, except that $S$
uses the last three internal returns of $Z$ rather than the first three. The
shortcut is still hidden from the root frame, but it is more local in the flat
trace because the relevant returns occur near the answer-generation region. This
tests whether CoT follows a shortcut even when the hidden signal is close to the
final answer rather than buried early in the trace.

\paragraph{Position-weighted aggregate.}
Again using the root template $\texttt{add}(Z,Y_{\mathrm{comp}})$, the shortcut
signal is
\begin{equation}
S(\mathbf{x})=\left(\sum_{t:r_t=0} t\right)\bmod 10,
\end{equation}
where $r_t$ is the internal return value at trace position $t$ inside $Z$.
Training chooses the compensating argument so that the answer equals this
position-weighted aggregate; OOD resamples the argument and breaks the
relationship. This case makes the shortcut depend on both values and global
positions in the flat trace.

\paragraph{Validity conditions.}
\label{app:prop-conditions}
The constructions are designed to isolate shortcut visibility rather than
artifacts of the generator. Two conditions are important. First, the intended
root-frame computation must be identifiable from RM's training view; otherwise
RM could fail for reasons unrelated to hidden shortcuts. Second, the root-frame
payload must not give a one-to-one code for $S$; otherwise RM could recover the
shortcut without seeing the hidden trace. We also exclude variants where $S$
appears as a literal in the static root expression. The six constructions in
\Cref{tab:robustness} satisfy these conditions.

Two ablations motivated these checks. In sibling-copy variants, the root frame
itself underdetermines which child value should be returned, and RM can fail
even though the shortcut is not visible. In one-to-one wrapper variants of Case
1, the root payload contains an invertible code for the hidden value $z$, so RM
can learn the same shortcut from its own view. We exclude both types because
they no longer isolate the difference between full-trace visibility and
active-frame visibility.